\pdfoutput=1

\relax
\documentclass[letterpaper]{article} 
\usepackage[preprint]{my_style} 
\usepackage{times}  
\usepackage{helvet} 
\usepackage{courier}  
\usepackage[hyphens]{url}  
\usepackage{graphicx} 
\usepackage{graphicx}  
\usepackage{verbatim}
\usepackage{amsmath,bm}
\usepackage{amsthm}
\usepackage{amssymb}
\usepackage{verbatim}
\usepackage{algorithm}
\usepackage{algpseudocode}
\usepackage{subcaption}
\usepackage{multirow}
\usepackage{verbatim}
\usepackage{enumitem}
\usepackage[hang,flushmargin]{footmisc}
\usepackage{lipsum}
\usepackage{booktabs}
\DeclareMathOperator*{\diag}{Diag\,}
\newtheorem{lemma}{Lemma}
\newtheorem{theorem}{Theorem}
\newtheorem{corollary}{Corollary}
\newtheorem{assumption}{Assumption}
\newtheorem{definition}{Definition}
\newtheorem{remark}{Remark}

\title{D-SPIDER-SFO: A Decentralized Optimization Algorithm with Faster Convergence Rate for Nonconvex Problems}
\author{\textbf{Taoxing Pan\textsuperscript{\rm 1}, Jun Liu\textsuperscript{\rm 2}, Jie Wang\textsuperscript{\rm 1}\thanks{Corresponding author}}\\ 
\textsuperscript{\rm 1} University of Science and Technology of China,
\textsuperscript{\rm 2} Infinia ML, Inc.\\
tx1997@mail.ustc.edu.cn\\
jun.liu@infiniaml.com\\
jiewangx@ustc.edu.cn 
}
\begin{document}

\maketitle

\begin{abstract}
    Decentralized optimization algorithms have attracted intensive interests recently, as it has a balanced communication pattern, especially when solving large-scale machine learning problems. Stochastic Path Integrated Differential Estimator Stochastic First-Order method (SPIDER-SFO) nearly achieves the algorithmic lower bound in certain regimes for nonconvex problems. However, whether we can find a decentralized algorithm which achieves a similar convergence rate to SPIDER-SFO is still unclear. To tackle this problem, we propose a decentralized variant of SPIDER-SFO, called \textbf{d}ecentralized SPIDER-SFO (D-SPIDER-SFO). We show that D-SPIDER-SFO achieves a similar gradient computation cost---that is,  $\mathcal{O}(\epsilon^{-3})$ for finding an $\epsilon$-approximate first-order stationary point---to its centralized counterpart. To the best of our knowledge, D-SPIDER-SFO achieves the state-of-the-art performance for solving nonconvex optimization problems on decentralized networks in terms of the computational cost. Experiments on different network configurations demonstrate the efficiency of the proposed method.
\end{abstract}

\section{Introduction}

Distributed optimization is a popular technique for solving large scale machine learning problems \cite{li2014scaling}, ranging from visual object recognition \cite{huang2017densely,he2016deep} to natural language processing \cite{vaswani2017attention,devlin2018bert}. For distributed optimization, a set of workers form a connected computational network, and each worker is assigned a portion of the computing task. The centralized network topology, like parameter server \cite{chen2016revisiting,dean2012large,li2014scaling,NIPS2010_4006}, consists of a central worker connected with all other workers. This communication mechanism could degrade the performance significantly in scenarios where the underlying network has low bandwidth or high latency \cite{lian2017can}.

In contrast, the decentralized network topology offers better network load balance---as all nodes in the network only communicate with their neighbors instead of the central node---which implies that they may be able to outperform their centralized counterparts. These motivate many works on decentralized algorithms. \citet{nedic2009distributed} studied distributed subgradient method for optimizing a sum of convex objective functions. \citet{shi2014admm} analyzed the linear convergence rate of the ADMM in decentralized consensus optimization. \citet{yuan2016convergence} studied the convergence properties of the decentralized gradient descent method (DGD). They proved that the local solutions and the mean solution converge to a neighborhood of the global minimizer at a linear rate for strongly convex problems. \citet{mokhtari2016dsa} studied decentralized double stochastic averaging gradient algorithm (DSA) and \citet{shi2015extra} proposed decentralized exact first-order algorithm (EXTRA). Both of these two algorithms converge to an optimal solution at a linear rate for strongly convex problems. \citet{lian2017can} studied decentralized PSGD (D-PSGD) and showed that decentralized algorithms could be faster than their centralized counterparts. \citet{tang2018d} proposed D$^2$ algorithm which is less sensitive to the data variance across workers. \citet{scaman2019opt} provided two optimal decentralized algorithms, called multi-step primal-dual (MSPD) and distributed randomized smoothing (DRS), and their corresponding optimal convergence rate for convex problems in certain regimes. \citet{assran2019stochastic} proposed Stochastic Gradient Push (SGP) and proved that SGP converges to a stationary point of smooth and nonconvex objectives at the sub-linear rate.

On the other hand, to achieve a faster convergence rate, researchers have also proposed many nonconvex optimization algorithms.
Stochastic Gradient Descent (SGD) \cite{robbins1951stochastic} achieves an $\epsilon$-approximate stationary point with a gradient cost of $\mathcal{O}(\epsilon^{-4})$ \cite{ghadimi2013stochastic}.
To improve the convergence rate of SGD, researchers have proposed variance-reduction methods \cite{roux2012stochastic,defazio2014saga}.  
Specifically, the finite-sum Stochastic Variance Reduced Gradient method (SVRG) \cite{johnson2013accelerating,Reddi2016} and online Stochastically Controlled Stochastic Gradient method (SCSG) \cite{lei2017non} achieve a gradient cost of $\mathcal{O}(\min(m^{2/3}\epsilon^{-2}, \epsilon^{-10/3}))$, where $m$ is the number of samples.
SNVRG \cite{zhou2018stochastic} achieves a gradient cost of $\mathcal{\widetilde{O}}(\epsilon^{-3})$, while SPIDER-SFO \cite{fang2018spider} and SARAH \cite{nguyen2017s,nguyen2019f} achieve a gradient cost of $\mathcal{O}( \epsilon^{-3})$. Moreover, \citet{fang2018spider} showed that SPIDER-SFO nearly achieves the algorithmic lower bound in certain regimes for nonconvex problems.  Though these works have made significant progress, convergence properties of faster optimization algorithms for nonconvex problems in \emph{the decentralized settings} are unclear.

In this paper, we propose \textbf{d}ecentralized SPIDER-SFO (D-SPIDER-SFO) for faster convergence rate for nonconvex problems. We theoretically analyze that D-SPIDER-SFO achieves an $\epsilon$-approximate stationary point in gradient cost of $\mathcal{O}(\epsilon^{-3})$, which achieves the state-of-the-art performance for solving nonconvex optimization problems in the decentralized settings. Moreover, this result indicates that D-SPIDER-SFO achieves a similar gradient computation cost to its centralized competitor, called centralized SPIDER-SFO (C-SPIDER-SFO). To give a quick comparison of our algorithm and other existing first-order algorithms for nonconvex optimization in the decentralized settings, we summarize the gradient cost and communication complexity of the most relevant algorithms in Table\ref{table:computing_cost_comparison}. Table \ref{table:computing_cost_comparison} shows that D-SPIDER-SFO converges faster than D-PSGD and D$^2$ in terms of the gradient computation cost. Moreover, compared with C-SPIDER-SFO, D-SPIDER-SFO reduces much communication cost on the busiest worker. Therefore, D-SPIDER-SFO can outperform C-SPIDER-SFO when the communication becomes the bottleneck of the computational network. Our main contributions are as follows.
\begin{enumerate}
    \item We propose D-SPIDER-SFO for finding approximate first-order stationary points for nonconvex problems in the decentralized settings, which is  a decentralized parallel version of SPIDER-SFO.
    \item We theoretically analyze that D-SPIDER-SFO achieves the gradient computation cost of $\mathcal{O}( \epsilon^{-3})$ to find an $\epsilon$-approximate first-order stationary point, which is similar to SPIDER-SFO in the centralized network topology. To the best of our knowledge, D-SPIDER-SFO achieves the state-of-the-art performance for solving nonconvex optimization problems in the decentralized settings.
\end{enumerate}
\textbf{Notation}: Let $\|\cdot\|$ be the vector and the matrix $\ell_2$ norm and $\|\cdot\|_F$ be the matrix Frobenius norm. $\nabla f(\cdot)$ denotes the gradient of a function $f$. Let $\mathbf{1}_n$ be the column vector in $\mathbb{R}^n$ with $1$ for all elements and $e_i$ be the column vector with a 1 in the $i$th coordinate and 0's elsewhere. We denote by $f^*$ the optimal solution of $f$. For a matrix $A\in \mathbb{R}^{n\times n}$, let $\lambda_i(A)$ be the $i$-th largest eigenvalue of a matrix. For any fixed integer $j\geq i\geq 0$, let $[i:j]$ be the set $\{i,i+1, \dots, j\}$ and $\{x\}_{i:j}$ be the sequence $\{x_i, x_{i+1}, \dots, x_j\}$.

 \begin{table*}
     \centering
     \caption{Comparision of D-PSGD, D$^2$ and D-SPIDER-SFO and their centralized competitors.}
     \begin{tabular}{c|c|c|c}
         \toprule
         {Algorithm} & {Communication cost on} &{Gradient}&{Bounded data}\\
         &{the busiest node} & {Computation Cost} & {variance among workers}\\
         \hline 
         {C-PSGD \cite{dekel2012optimal}} & {$\mathcal{O}(n)$} &  {$\mathcal{O}(\epsilon^{-4})$}& {$\times$}\\
         \hline 
         {D-PSGD \cite{lian2017can}} & {$\mathcal{O}$ $($Deg$($network$))$}  & {$\mathcal{O}(\epsilon^{-4})$}& {need}\\
         \hline
         {D$^2$\cite{tang2018d}} & {$\mathcal{O}$ $($Deg$($network$))$}  & {$\mathcal{O}(\epsilon^{-4})$}& {no need} \\
         \hline
         {C-SPIDER-SFO \cite{fang2018spider}} & {$\mathcal{O}(n)$} & {$\mathcal{O}( \epsilon^{-3})$}& {$\times$}\\
         \hline 
         {D-SPIDER-SFO} & {$\mathcal{O}$ $($Deg$($network$))$} & {$\mathcal{O}( \epsilon^{-3})$}& {no need}\\
         \bottomrule
     \end{tabular}
     \label{table:computing_cost_comparison}
 \end{table*}

\section{Basics and Motivation}

\subsection{Decentralized Optimization Problems}
In this section, we briefly review some basics of the decentralized optimization problem. We represent the decentralized communication topology with a weighted directed graph: $(V,W)$. $V$ is the set of all computational nodes, that is, $V:=\{1,2,\dots,n\}.$ $W$ is a matrix and $W_{i,j}$ represents how much node $j$ can affect node $i$, while $W_{ij}=0$ means that node $i$ and $j$ are disconnected. Therefore, $W_{ij} \in [0,1],$ for all $i,j$. Moreover, in the decentralized optimization settings, we assume that $W$ is symmetric and doubly stochastic, which means that $W$ satisfies (i) $W_{ij} = W_{ji}$ for all $i,j$, and (ii) $\sum_{j} W_{ij}=1$ for all $i$ and $\sum_{i} W_{ij} =1$ for all $j$.

Throughout this paper, we consider the following decentralized optimization problem:
\begin{align}\label{problem}
\min_{x\in \mathbb{R}^N} f(x) :=\frac{1}{n}\sum_{i=1}^{n} \underbrace{\mathbb{E}_{\xi\in \mathcal{D}_i} F_i(x; \xi)}_{=:f_i(x)},
\end{align}
where $n$ is the number of workers, $\mathcal{D}_i$ is a predefined distribution of the local data for worker $i$, and $\xi$ is a random data sample. Decentralized problems require that the graph of the computational network is connected and each worker can only exchange local information with its neighbors.

In the $i$-th node, $x_i, \xi_i, f_i(x_i), F_i(x_i;\xi_i)$ is the local optimization variables, random sample, target function and stochastic component function. 
Let $\mathcal{S}$ be a subset that samples $S$ elements in the dataset. For simplicity, we denote by $\xi_{k,i}$ the subset that $i$-th node samples at iterate $k$, that is, $\nabla F_i(x_{k,i};\xi_{k,i}) = \nabla F_{i}(x_{k,i};\mathcal{S}_{k,i}) = \frac{1}{S_{k,i}}\sum_{\xi_j\in \mathcal{S}_{k,i}} \nabla F_i(x_{k,i}; \xi_j)$. In order to present the core idea more clearly, at iterate $k$, we define the concatenation of all local optimization variables, estimators of full gradients, stochastic gradients, and full gradients by matrix $X_k, G_k,\partial F(X_k;\xi_k), \partial f(X_k) \in \mathbb{R}^{N\times n}$ respectively:
\begin{align*}
    X_k & :=[x_{k,1},\ \ \cdots \ \ ,x_{k,n}],\\
    G_k & := [g_{k,1},\ \ \cdots\ \ ,g_{k,n}],\\
    \partial F(X_k,\xi_k) & :=[\nabla F_1(x_{k,1};\xi_{k,1}),\ \ \cdots\ \  ,\nabla F_n(x_{k,n};\xi_{k,n})],\\
    \partial f(X_k) & :=[\nabla f_1(x_{k,1}),\ \ \cdots\ \ ,\nabla f_n(x_{k,n})].
\end{align*}

In general, at iterate $k$, let the stepsize be $\eta_k$. We define $\eta_k V_k$ as the update, where $V_k \in \mathbb{R}^{N\times n}$. Therefore, we can view the update rule as:
\begin{align}\label{general_update_rule}
X_{k+1} \gets X_k W  - \eta_k V_k.
\end{align}

\section{D-SPIDER-SFO}

In this section, we introduce the basic settings, assumptions, and the flow of D-SPIDER-SFO in the first subsection. Then, we compare D-SPIDER-SFO with D-PSGD and D$^2$ in a special scenario to show our core idea. In the final subsection, we propose the error-bound theorems for finding an $\epsilon$-approximate first-order stationary point.
\begin{algorithm}
    \caption{D-SPIDER-SFO on the $i$th node}
    \label{alg:my_algorithm}
    \begin{algorithmic}
        \State {\bfseries Input:} Require initial point $X_0$,  weighted matrix $W$, number of iterations $K$, learning rate $\eta$, constant $q$, and two sample sizes $S^{(1)}$ and $S^{(2)}$\\
        {\bfseries Initialize:} $X_{-1} = X_0$, $G_{-1}$ = $0$
        \For{$k = 0, \dots, K-1$}
        \If {mod$(k,q)=0$}
        \State Draw $S^{(1)}$ samples and compute the stochastic gradient $\nabla F_{i}(x_{k,i};\mathcal{S}^{(1)}_{k,i})$ 
        \State $g_{k,i} = \nabla F_{i}(x_{k,i};\mathcal{S}^{(1)}_{k,i})$
        \State $x_{k+\frac{1}{2}, i} = 2x_{k,i} - x_{k-1,i} - \eta (g_{k,i} - g_{k-1,i})$
        \Else
        \State Draw $S^{(2)}$ samples, and compute two stochastic gradient $\nabla F_{i}(x_{k,i};\mathcal{S}^{(2)}_{k,i})$ and $\nabla F_{i}(x_{k-1,i};\mathcal{S}^{(2)}_{k,i})$
        \State $g_{k,i} = \nabla F_{i}(x_{k,i};\mathcal{S}^{(2)}_{k,i}) - \nabla F_{i}(x_{k-1,i};\mathcal{S}^{(2)}_{k,i}) + g_{k-1,i}$
        \State $x_{k+\frac{1}{2}, i} = 2x_{k,i} - x_{k-1,i} -\eta (\nabla F_{i}(x_{k,i};\mathcal{S}^{(2)}_{k,i}) - \nabla F_{i}(x_{k-1,i};\mathcal{S}^{(2)}_{k,i}))$
        \EndIf
        \State $x_{k+1,i}=\sum_{j=1}^n W_{j,i} x_{k+\frac{1}{2},j}$
        \EndFor
        \State Return $\widetilde{x} = \frac{X_K\mathbf{1}_n}{n}$ 
    \end{algorithmic}
\end{algorithm}

\subsection{Settings and Assumptions}

In this subsection, we introduce the formal definition of an $\epsilon$-approximate first-order stationary point and commonly used assumptions for decentralized optimization problems. Moreover, we briefly introduce the key steps at iterate $k$ for worker $i$ in D-SPIDER-SFO algorithm.

\begin{definition}
    We call $\widetilde{x}\in \mathbb{R}^N$ an $\epsilon$-approximate first-order stationary point, if 
    \begin{align}
    \|\nabla f(\widetilde{x})\|\leq \epsilon.
    \end{align}
\end{definition}

\begin{assumption} \label{assumption_SPIDER}
    We make the following commonly used assumptions for the convergence analysis.
    \begin{enumerate}
        \item \textbf{Lipschitz gradient: } All local loss functions $f_i(\cdot)$ have $L$-Lipschitzian gradients.
        \item \textbf{Average Lipschitz gradient: } In each fixed node $i$, the component function $F_i(x_i;\xi_i)$ has an average L-Lipschitz gradient, that is,
        $$\mathbb{E}\|\nabla F_i(x;\xi_i) - \nabla F_i(y;\xi_i)\|^2\leq L^2\|x - y\|^2,\forall x,y.$$
        \item \textbf{Spectral gap:} Given the symmetric doubly stochastic matrix $W$. Let the eigenvalues of $W\in \mathbb{R}^{n\times n}$ be $\lambda_1\geq \lambda_2\geq \cdots \geq \lambda_n$. We denote by $\lambda$ the second largest value of the set of eigenvalues, i.e., $$\lambda = \max_{i\in \{2,\cdots,n\}} \lambda_i = \lambda_2.$$ We assume $\lambda<1$ and $\lambda_n >-\frac{1}{3}$.
        \item \textbf{Bounded variance: } Assume the variance of stochastic gradient within each worker is bounded, which implies there exists a constant $\sigma$, such that
        $$\mathbb{E}_{\xi\sim \mathcal{D}_i}\|\nabla F_i(x;\xi) - \nabla f_i(x)\|^2\leq \sigma^2, \forall i, \forall x.$$
        \item (For D-PSGD Algorithm only) \textbf{Bounded data variance among workers: } Assume the variance of full gradient among all workers is bounded, which implies that there exists a constant $\zeta$, such that
        $$\mathbb{E}_{i \sim \mathcal{U}([n])}\|\nabla f_i(x) - \nabla f(x)\|^2\leq \zeta^2, \forall i, \forall x. $$
    \end{enumerate}
\end{assumption}
\begin{remark}
    The eigenvalues of $W$ measure the speed of information spread across the network \cite{lian2017can}. D-SPIDER-SFO requires $\lambda_2<1$ and $\lambda_n>-\frac{1}{3}$, which is the same as the assumption in D$^2$ \cite{tang2018d}, while D-PSGD only needs the former condition. D-PSGD needs bounded data variance among workers assumption additionally, as it is sensitive to such data variance.
\end{remark}

D-SPIDER-SFO algorithm is a synchronous decentralized parallel algorithm. Each node repeats these four key steps at iterate $k$ concurrently: 

\begin{enumerate}
    \item Each node computes a local stochastic gradient on their local data. When $\mod(k,q) \not= 0$, all nodes compute $\nabla F_{i}(x_{k,i};\mathcal{S}^{(2)}_{k,i})$ and $\nabla F_{i}(x_{k-1,i};\mathcal{S}^{(2)}_{k,i})$ using the local models at both iterate $k$ and the last iterate; otherwise, they compute $\nabla F_{i}(x_{k,i};\mathcal{S}^{(1)}_{k,i})$.
    \item Each node updates its local estimator of the full gradient $g_{k,i}$. When $\mod(k,q) \not= 0$, all nodes compute $g_{k,i} \gets \nabla F_{i}(x_{k,i};\mathcal{S}^{(2)}_{k,i}) - \nabla F_{i}(x_{k-1,i};\mathcal{S}^{(2)}_{k,i}) + g_{k-1,i}$; else they compute $g_{k,i} \gets \nabla F_{i}(x_{k,i};\mathcal{S}^{(1)}_{k,i})$.
    \item Each node updates their local model. When $\mod(k,q) \not=0$, all nodes compute $x_{k+\frac{1}{2}, i} \gets 2x_{k,i} - x_{k-1,i} -\eta (\nabla F_{i}(x_{k,i};\mathcal{S}^{(2)}_{k,i}) - \nabla F_{i}(x_{k-1,i};\mathcal{S}^{(2)}_{k,i}))$; else they compute $x_{k+\frac{1}{2}, i} \gets 2x_{k,i} - x_{k-1,i} - \eta (g_{k,i} - g_{k-1,i})$.
    \item When meeting the synchronization barrier, each node takes weighted average with its and neighbors' local optimization variables:
    $x_{k+1,i}=\sum_{j=1}^n W_{j,i} x_{k+\frac{1}{2},j}.$
\end{enumerate}

To understand D-SPIDER-SFO, we consider the update rule of global optimization variable $\frac{X_k\mathbf{1}_n}{n}$. Let $k_0 = \lfloor k/q\rfloor \cdot q$. For convenience, we define 
\begin{align*}
    \overline{\Delta}_k &= \frac{(X_{k+1} - X_k)\mathbf{1}_n}{n} = \frac{1}{n} \sum_{i=1}^n (x_{k+1,i} - x_{k,i}),\\
    \overline{H}_k(X) &= \partial F(X;\xi_k)\frac{\mathbf{1}_n}{n}=\frac{1}{n} \sum_{i=1}^n \nabla F_i(x_{i};\xi_{k,i}),
\end{align*}
where $\xi_k$ denotes the samples at the $k$-th iterate. Therefore,
\begin{align*}
    \overline{\Delta}_{k} 
    =&\overline{\Delta}_{k-1} - \eta(\overline{H}_k(X_k) - \overline{H}_{k}(X_{k-1}))\notag\\
    =&-\eta \overline{H}_{k_0}(X_{k_0}) - \eta \sum_{s=k_0+1}^k(\overline{H}_s(X_s) - \overline{H}_{s}(X_{s-1})).
\end{align*}
As for centralized SPIDER-SFO, we have
\begin{align*}
    x_{k+1} - x_{k}
    =& -\eta_k(\nabla F(x_k;\xi_k) - \nabla F(x_{k-1}; \xi_{k}) + v_{k-1})\notag\\
    =& - \eta_{k_0} \nabla F(x_{k_0};\xi_{k_0}) \notag - \sum_{s=k_0+1}^k \eta_{s}(\nabla F(x_s;\xi_s) - \nabla F(x_{s-1};\xi_{s})).
\end{align*}

\begin{remark}
    Nguyen et al. propose SARAH for (strongly) convex optimization problems. SPIDER-SFO adopts a similar recursive stochastic gradient update framework and nearly matches the algorithmic lower bound in certain regimes for nonconvex problems. Moreover, Wang et al. [2] propose SpiderBoost and show that SpiderBoost, a variant of SPIDER-SFO with fixed step size, achieves a similar convergence rate to SPIDER-SFO for nonconvex problems. Inspired by these algorithms, we propose decentralized SPIDER-SFO (D-SPIDER-SFO). As we can see, the update rule of D-SPIDER-SFO is similar to its centralized counterpart with fixed step size.
\end{remark}

\subsection{Core Idea}

The convergence property of decentralized parallel stochastic algorithms is related to the variance of stochastic gradients and the data variance across workers. In this subsection, we present in detail the underlying idea to reduce the gradient complexity behind the algorithm design.

The general update rule \eqref{general_update_rule} shows that $\mathbb{E}[\left\|V_k\right\|^2_F]$ affects the convergence, especially when we approach a solution. For showing the improvement of D-SPIDER-SFO, we will compare the upper bound of $\mathbb{E}[\|V_k\|^2_F]$ of three algorithms, which are D-PSGD, D$^2$, and D-SPIDER-SFO.

The update rule of D-PSGD is $X_{k+1} = X_{k} W - \eta \partial F(X_k; \xi_k)$, that is, $V_k = \partial F(X_k; \xi_k)$. Then, we have
{ 
\begin{align*}
    \mathbb{E}\|\partial F(X_k; \xi_k)\|_F^2 \leq &4\mathbb{E}\|\partial F(X_k;\xi_k) - \partial f(X_k)\|_F^2 + 4\mathbb{E}\left\|\partial f(X_k) - \partial f\left(\frac{X_k\mathbf{1}_n}{n} \mathbf{1}_n^T\right)\right\|_F^2 \\
    &+ 4\mathbb{E}\left\|\partial f\left(\frac{X_k\mathbf{1}_n}{n} \mathbf{1}_n^T\right) - \nabla f\left(\frac{X_k\mathbf{1}_n}{n}\right)\mathbf{1}_n^T\right\|_F^2 + 4\mathbb{E}\left\|\nabla f\left(\frac{X_k\mathbf{1}_n}{n}\right)\mathbf{1}_n^T\right\|_F^2\\
    \leq & 4n\sigma^2 + 4n\zeta^2 +4L^2\sum_{i=1}^n\left\|x_{k,i} - \frac{X_k\mathbf{1}_n}{n}\right\|^2 + 4\mathbb{E}\left\|\nabla f\left(\frac{X_k\mathbf{1}_n}{n}\right)\mathbf{1}_n^T\right\|_F^2.
\end{align*}
}
Moreover, the update rule of D$^2$ is $X_{k+1} = \left[2 X_k - X_{k-1} - \eta(\partial F(X_k;\xi_k) - \partial F(X_{k-1};\xi_{k-1}))\right] W$. For convenience, we define $Q_k = \frac{X_k - X_{k-1}}{\eta}$. Therefore, we can conclude the upper bound of $\mathbb{E}\|V_k\|_F^2$.
 \begin{align*}
    \mathbb{E}\left\|V_k\right\|^2_F = &\mathbb{E}\left\|\left[-Q_k + (\partial F(X_k; \xi_k) - \partial F(X_{k-1};\xi_{k-1}))\right]W\right\|^2_F\\
    \leq & 2 \mathbb{E} \left\|Q_k \right\|^2_F
    + 2\mathbb{E}\|\partial F(X_k; \xi_k) - \partial F(X_{k-1};\xi_{k-1})\|^2_F\\
    \leq & 2 \mathbb{E} \left\|Q_k\right\|^2_F + 6 \mathbb{E}\|\partial F(X_k;\xi_k) - \partial f(X_k)\|^2_F\\
    &+ 6 \mathbb{E}\|\partial F(X_{k-1};\xi_{k-1}) - \partial f(X_{k-1})\|^2_F + 6\mathbb{E}\|\partial f(X_{k}) - \partial f(X_{k-1})\|^2_F\\
    \leq & \frac{2}{\eta^2} \sum_{i=1}^n\mathbb{E} \left\|x_{k,i} - x_{k-1,i}\right\|^2 + 6\sum_{i=1}^n \mathbb{E}\|\nabla F_i(x_{k,i}; \xi_{k,i}) - \nabla f_i(x_{k,i})\|^2\\
    &+6\sum_{i=1}^n \mathbb{E}\|\nabla F_i(x_{k-1,i}; \xi_{k-1,i}) - \nabla f_i(x_{k-1,i})\|^2 +6\sum_{i=1}^n \mathbb{E}\|\nabla f_i(x_{k,i}) - \nabla f_i(x_{k-1,i})\|^2\\
    \leq & 2\left(\eta^{-2} + 3L^2 \right)\sum_{i=1}^n \mathbb{E}\|x_{k,i} - x_{k-1,i}\|^2 + 12n\sigma^2.
\end{align*}

Since the update rule of D-SPIDER-SFO has two different patterns, we discuss them seperately. If $\mod(k,q)\not=0$, we have $V_k = [-Q_k - (\partial F(X_k; \xi_k) - \partial F(X_{k-1};\xi_{k}))]W$. 
\begin{align*}
    \mathbb{E}\left\|V_k\right\|^2_F = &\mathbb{E}\left\|\left[Q_k + (\partial F(X_k; \xi_k) - \partial F(X_{k-1};\xi_{k}))\right]W\right\|^2_F\\
    \leq & 2 \mathbb{E} \left\|Q_k \right\|^2_F + 2\mathbb{E}\|\partial F(X_k; \xi_k) - \partial F(X_{k-1};\xi_{k})\|^2_F\\
    \leq & \frac{2}{\eta^2} \sum_{i=1}^n\mathbb{E} \left\|x_{k,i} - x_{k-1,i}\right\|^2 + 2\sum_{i=1}^n \mathbb{E}\|\nabla F_i(x_{k,i}; \xi_{k,i}) - \nabla F_i(x_{k-1,i};\xi_{k,i})\|^2_F\\
    \leq & 2\left(\eta^{-2} + L^2 \right)\sum_{i=1}^n \mathbb{E}\|x_{k,i} - x_{k-1,i}\|^2.
\end{align*}
If $\mod(k,q) = 0$ and $k>0$, we have $V_k = [-Q_k - (\partial F(X_k; \xi_k) - G_{k-1})]W$. Let $k_0 = k - q$, and we have 
\begin{align} \label{ineq_v_k_of_D-SPIDER}
    \mathbb{E}\left\|V_k\right\|^2_F = &\mathbb{E}\left\|\left[-Q_k - (\partial F(X_k; \xi_k) - G_{k-1})\right]W\right\|^2_F \notag\\
    \leq &2\mathbb{E}\|Q_k\|^2_F + 2\mathbb{E}\|\partial F(X_k; \xi_k) - G_{k-1}\|^2_F \notag\\
    \leq &2\mathbb{E}\|Q_k\|^2_F+ 4\mathbb{E}\|\partial F(X_k; \xi_k) - \partial F(X_{k_0};\xi_{k_0})\|^2_F +4\mathbb{E}\left\|\sum_{j=k_0}^{k-2}(\partial F(X_{j+1}; \xi_{j+1}) - \partial F(X_{j};\xi_{j+1}) \right\|^2_F,
\end{align}
where 
\begin{align*}
    G_{k-1} = & \sum_{j=k_0}^{k-2}  \left( \partial F(X_{j+1}; \xi_{j+1}) - \partial F(X_{j};\xi_{j+1})\right) + \partial F(X_{k_0};\xi_{k_0}).
\end{align*}

Assume that for any $j\in [k_0:k]$, $X_j$ has achieved the optimum $X^*:= x^*\mathbf{1}_n^T$ with all local models equal to the optimum $x^*$. Then, $\mathbb{E}[\left\|V_k\right\|_F^2]$ of D-PSGD, and D$^2$, is bounded by $\mathcal{O}(\sigma^2 + \zeta^2)$, $\mathcal{O}(\sigma^2)$, which is similar to \citet{tang2018d}. For convenience, considering the finite-sum case, if we set the batch size $S_1$ equal to the size $m$ of the dataset, that is, we compute the full gradient at iteration $k$ and $k_0$. Moreover, as for any $j\in [k_0: k]$, $X_j=X^*$, then each term of \eqref{ineq_v_k_of_D-SPIDER} is zero, that is, $\mathbb{E}[\left\|V_k\right\|_F^2]$ is bounded by zero. 
D-SPIDER-SFO will stop at the optimum, while D-PSGD and D$^2$ will escape from the optimum because of the variance of stochastic gradients or data variance across workers. If we need D$^2$ stops at the optimum, D$^2$ should compute the full gradient at each iteration, which is similar to EXTRA \cite{shi2015extra}, while D-SPIDER-SFO needs to compute full gradient per $q$ iteration. This is the key ingredient for the superior performance of D-SPIDER-SFO. By this sight, D-SPIDER-SFO achieves a faster convergence rate. In the following analysis, we show that the gradient cost of D-SPIDER-SFO is $\mathcal{O}(\frac{1}{\epsilon^3})$.

\subsection{Convergence Rate Analysis}

In this subsection, we analyze the convergence properties of the D-SPIDER-SFO algorithm. We propose the error bound of the gradient estimation in Lemma \ref{lemma:variance_of_vk}, which is critical in convergence analysis. Then, based on Lemma \ref{lemma:variance_of_vk}, we present the upper bound of gradient cost for finding an $\epsilon$ approximate first-order stationary point, which is the state-of-the-art for decentralized nonconvex optimization problems.

Before analyzing the convergence properties, we consider the update rule of global optimization variables as follows,
$$\frac{X_{k+1}\mathbf{1}_n}{n} = \frac{(X_k W - \eta V_k)\mathbf{1}_n}{n}= \frac{(X_k - \eta V_k)\mathbf{1}_n}{n}.$$

To analyze the convergence rate of D-SPIDER-SFO, we conclude the following Lemma \ref{lemma:variance_of_vk} which bounds the error of the gradient estimator $\frac{V\mathbf{1}_{n}}{n}$.

\begin{lemma}\label{lemma:variance_of_vk}
    Under the Assumption \ref{assumption_SPIDER}, we have
   	\begin{align*}
	    &\frac{1}{K}\sum_{k=0}^{K-1} \mathbb{E} \left\|\frac{\partial f(X_k) \mathbf{1}_n}{n} - \frac{V_k \mathbf{1}_n}{n} \right\|^2\\
	    \leq& \frac{12C_1L^2 q}{KnDS_2}\mathbb{E}\|X_1\|^2_F + \left(\frac{72C_2\eta^4L^4q^2}{KDS^2_2} + \frac{3qL^2\eta^2}{KS_2}\right)\sum_{k=1}^{K-1} \mathbb{E}\left\|\frac{V_{k-1} \mathbf{1}_n}{n} \right\|^2 + \left(1+ \frac{192C_2L^2\eta^2}{nDS_1}\right)\frac{\sigma^2}{S_1}.
	\end{align*}
    where 
    \begin{align*} C_1 &= \max\left\{\frac{1}{1-|b_n|^2},\frac{1}{(1-\lambda_2)^2}\right\},\\
        C_2 &= \max\left\{\frac{\lambda_n^2}{(1-|b_n|^2)}, \frac{\lambda_2^2}{(1-\sqrt{\lambda_2})^2(1-\lambda_2)}\right\},\\
        b_n &= \lambda_n - \sqrt{\lambda_n^2 - \lambda_n},\\
        D &= 1- \frac{48C_2q\eta^2L^2}{S_2}.
    \end{align*}
\end{lemma}

In Appendix, we will give the upper bound of $\mathbb{E}\|X_1\|^2_F$. Lemma \ref{lemma:variance_of_vk} shows that the error bound of the gradient estimator is related to the second moment of $\left\|\frac{X_k\mathbf{1}_n}{n}\right\|$. Then, we give the analysis of the convergence rate. W.l.o.g., we assume the algorithm starts from $0$, that is $X_0 = 0$, and define  $ \zeta_0 = \frac{1}{n}\sum_{i=1}^n\|\nabla f_i(\mathbf{0}) - \nabla f(\mathbf{0})\|$.

\begin{theorem}\label{3_SPIDER}
    For the online case, set parameters $S_1$, $S_2$, $\eta$, and $q$ as constants, and $C_1, C_2$, and $D$ as in Lemma \ref{lemma:variance_of_vk}. Then, under the Assumption \ref{assumption_SPIDER}, for Algorithm \ref{alg:my_algorithm}, we have 
    \begin{align*}
    &\frac{1}{K}\sum_{k=1}^K \mathbb{E}\left\|\nabla f\left(\frac{X_k \mathbf{1}_n}{n}\right)\right\|^2 + \frac{M}{K}\sum_{k=0}^{K-1}\mathbb{E}\left\| \frac{V_k \mathbf{1}_n}{n}\right\|^2\\
    \leq & \frac{2\mathbb{E}[f(\frac{X_0 \mathbf{1}_n}{n}) - f^*]}{\eta K} +  \left(1 + \frac{32C_2 L^2 \eta^2}{n q  D} + \frac{192C_2 L^2 \eta^2}{n S_2 D}\right)\frac{2\sigma^2}{S_1} + \frac{3\eta^2}{K}\left(\frac{4L^2C_1}{D} +  \frac{24L^2C_1 q}{D S_2}\right) (\sigma^2 + \zeta_0^2 + \|\nabla f(0)\|^2),
\end{align*}
    where 
    \begin{align*}
        &M :=1 - L\eta -\frac{6qL^2\eta^2}{S_2}\left[1+ \frac{4C_2L^2\eta^2}{D}\left(1+\frac{6q}{S_2}\right)\right].
    \end{align*}
\end{theorem}
By appropriately specifying the batch size $S_1, S_2$, the step size $\eta$, and the parameter $q$, we reach the following corollary. In the online learning case, we let the input parameters be
\begin{equation}\label{parameter_1}
    S_1 = \frac{\sigma^2}{\epsilon^2}, S_2 = \frac{\sigma}{\epsilon}, q = \frac{\sigma}{\epsilon},
\end{equation}
\begin{equation}\label{parameter_2}
    \eta<\min\left(\frac{-1+\sqrt{13}}{12 L}, \frac{1}{4\sqrt{3C_2}L}\right).
\end{equation}

\begin{corollary}\label{corollary: convergence rate}
    Set the parameters $S_1, S_2, q, \eta$ as in \eqref{parameter_1} and \eqref{parameter_2}, and set $K = \lfloor \frac{l}{\epsilon^2}\rfloor + 1$. 
 Then under the Assumption  \ref{assumption_SPIDER}, running Algorithm \ref{alg:my_algorithm} for $K$ iterations, we have
 \begin{align*}
 \frac{1}{K}\sum_{k=0}^{K-1}\mathbb{E} \left\|\nabla f\left(\frac{X_k\mathbf{1}}{n} \right)\right\|^2
 \leq 3\epsilon^2 + \frac{448 C_2L^2\eta^2\epsilon^3}{nD\sigma} ,
 \end{align*}

 where 
 \begin{align*}
     l := &\frac{2\mathbb{E}[f(0) -f^*]}{\eta} + \frac{84C_1L^2\eta^2}{D}(\sigma^2 + \zeta_0^2 + \|\nabla f(0)\|^2).
 \end{align*}
 The gradient cost is bounded by $2l\sigma\epsilon^{-3} + 2\sigma^2\epsilon^{-2}$.
\end{corollary}

\begin{remark}
    Corollary \ref{corollary: convergence rate} shows that measured by gradient cost, D-SPIDER-SFO achieves the convergence rate of $\mathcal{O}(\epsilon^{-3})$, which is similar to its centralized counterparts. Due to properties of decentralized optimization problems, the coefficient in Corollary \ref{corollary: convergence rate} of the term $\epsilon^{-3}$ depends on the network topology $W$ and the data variance among workers $\zeta^2$ in addition, while compared with the centralized competitor \cite{fang2018spider}. Although the differences exist, we conduct experiments to show that D-SPIDER-SFO converges with a similar speed to C-SPIDER-SFO.
\end{remark}

\section{Experiments}

In this section, we conduct extensive experiments to validate our theory. We introduce our experiment settings in the first subsection. Then in the second subsection, we conduct the experiments to demonstrate that D-SPIDER-SFO can get a similar convergence rate to C-SPIDER-SFO and converges faster than D-PSGD and D$^2$. Moreover, we validate that D-SPIDER-SFO outperforms its centralized counterpart, C-SPIDER-SFO, on the networks with low bandwidth or high latency. In the final, we show that D-SPIDER-SFO is robust to the data variance among workers. The code of D-SPIDER-SFO is available on GitHub at \url{https://github.com/MIRALab-USTC/D-SPIDER-SFO}.

\subsection{Experiment setting}

\textbf{Datasets and models}
We conduct our experiments on the image classification task. In our experiments, we train our models on CIFAR-10 \cite{krizhevsky2009learning}. The CIFAR-10 dataset consists of 60,000 32x32 color images in 10 classes when the training set has 50,000 images. For image classification, we train two convolution neural network models on CIFAR-10. The first one is LeNet5 \cite{Lecun98gradient-basedlearning}, which consists of a 6-filter $5\times 5$ convolution layer, a $2\times 2$ max-pooling layer, a 16-filter $5\times 5$ convolution layer and two fully connected layers with 120, 84 neurons respectively. The second one is ResNet-18 \cite{He2015DeepRL}.\\
\textbf{Implementations and setups} We implement our code on framework PyTorch. All implementations are compiled with PyTorch1.3 with gloo. We conduct experiments both on the CPU server and GPU server. CPU cluster is a machine with four CPUs, each of which is an Intel(R) Xeon(R) Gold 6154 CPU @ 3.00GHz with 18 cores. GPU server is a machine with 8 GPUs, each of which is a Nvidia GeForce GTX 2080Ti. In the experiments, we use the ring network topology, seeing each core or GPU as a node, with corresponding symmetric doubly stochastic matrix $W$ in the form of
\begin{align*}
    W = \left(
    \begin{matrix}
    1/2 &1/4 &    &    &    & 1/4\\
    1/4 &1/2 &1/4  &  &  &\\
    &1/4 &1/2 &\ddots &  &\\
    &  & \ddots & \ddots  & 1/4 &  &\\
    &  &  & 1/4 & 1/2 &1/4  \\
    1/4&  &  &  &1/4 &1/2
    \end{matrix}
    \right)\in \mathbb{R}^{n\times n}.
\end{align*}

\subsection{Experiments of D-SPIDER-SFO}

To show that D-SPIDER-SFO can get a similar convergence rate to its centralized version, we choose the computational complexity as metrics instead of the wall clock speed. In the experiments of training LeNet5, for D-PSGD and D$^2$, we use the constant learning rate $\frac{\eta_0}{\sqrt{K/n}}$ and tune $\eta_0$ from $\{0.1, 0.05, 0.01,0.005, 0.001\}$ and set the batch size $16$ for each node, where $K$ is the number of iterates and $n$ is the number of workers. For D-SPIDER-SFO and C-SPIDER-SFO, we set $S_1=256, S_2 = 16, q =16$ for each node, and tune the learning rate from $\{0.1,0.05,0.01,0.005,0.001\}$. When we conduct experiments on ResNet-18, for D-PSGD and D$^2$, we tune $\eta_0$ from $\{0.03, 0.01, 0.003, 0.001, 0.0003\}$, and also tune the learning rate of D-SPIDER-SFO and C-SPIDER-SFO from the same set $\{0.03, 0.01, 0.003, 0.001, 0.0003\}$. We conduct experiments on a computational network with eight nodes. Due to the space limitation, we show the experiments of training convolutional neural network models, LetNet5 and ResNet-18, on 8 GPUs in this paper and list the experiments on the CPU cluster in Supplement Material.

\begin{figure*}[t]
    \centering
        \includegraphics[width=1.00\columnwidth]{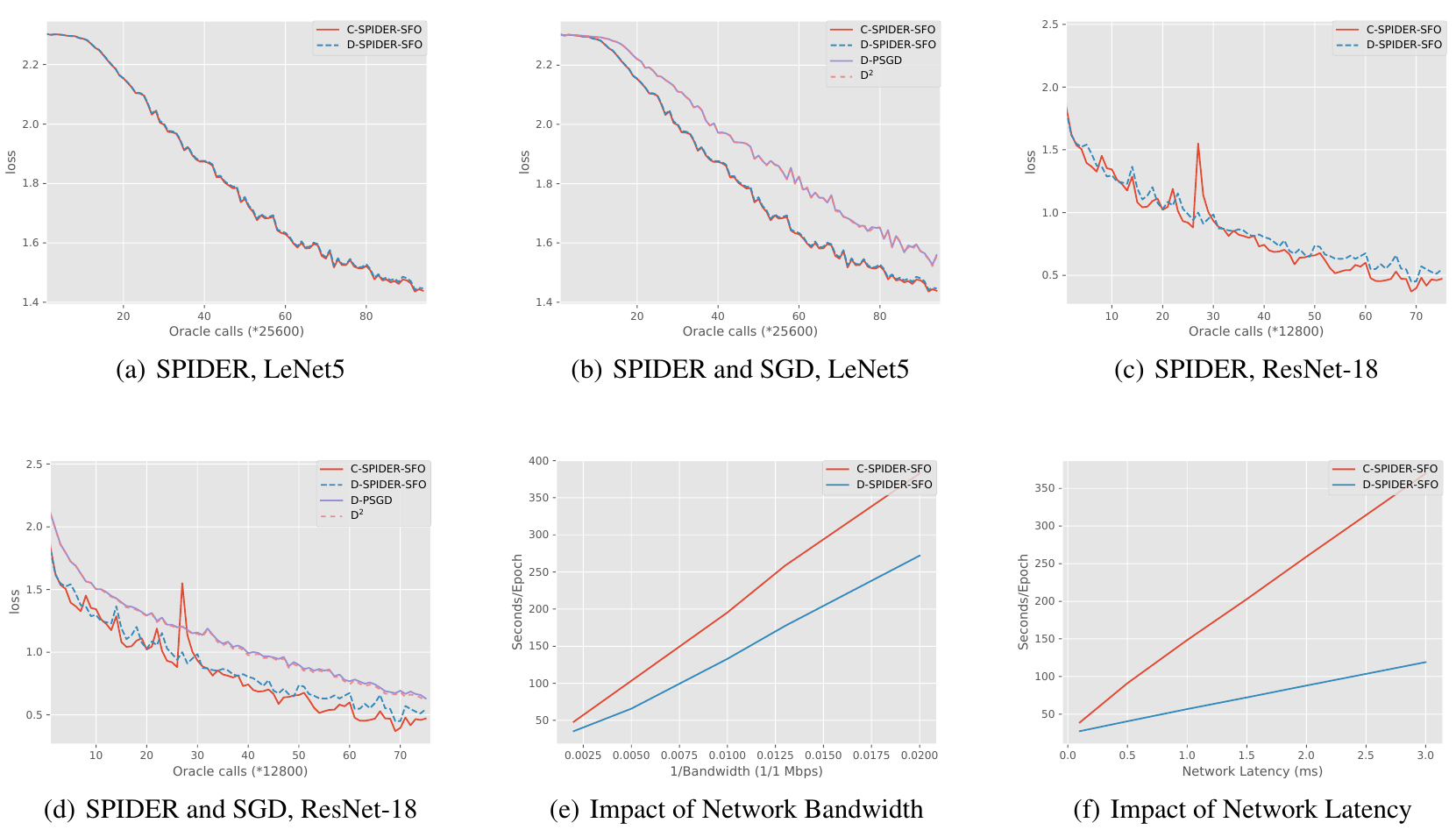}
        \caption{In the experiments, we train two convolutional neural network, LeNet5 and ResNet18. Fig. 1(a) and 1(c) are the comparisons between D-SPIDER-SFO and C-SPIDER-SFO. Fig. 1(b) and 1(d) are the comparisons between D-SPIDER-SFO, C-SPIDER-SFO, D-PSGD, and D$^2$. Fig. 1(e) and 1(d) show the impact of the bandwidth and latency.}
        \label{fig:all}
\end{figure*}

The gradient computation cost of both D-PSGD and D$^2$ is $\mathcal{O}(\epsilon^{-4})$ for finding an $\epsilon$ approximated stationary point, while D-SPIDER-SFO achieves $\mathcal{O}(\epsilon^{-3})$. Figure \ref{fig:all}(b) and \ref{fig:all}(d) validates our theoretical analysis and shows that D-SPIDER-SFO converges faster than D-PSGD and D$^2$.  Moreover, figure \ref{fig:all}(a) and \ref{fig:all}(c) also shows that D-SPIDER-SFO achieves a similar convergence rate to its centralized competitor.

As the decentralized network has more balanced communication patterns, D-SPIDER-SFO should outperform its centralized counterpart, when the communication becomes the bottleneck of the computational network. To demonstrate the above statement, we use the wall clock time as the metrics. In this experiment, we train LeNet5 on a cluster with 8 GPUs. We adopt the same parameters and experiment settings as what we use to train LeNet5. We use the tc command to control the bandwidth and latency of the network.  Figure \ref{fig:all}(e) and \ref{fig:all}(f) shows the wall clock time to finish one epoch on different network configurations. When the bandwidth becomes smaller, or the latency becomes higher, D-SPIDER-SFO can be even one order of magnitude faster than its centralized counterpart. The experiments demonstrate that the balanced communication pattern improves the efficiency of D-SPIDER-SFO.

\citet{tang2018d} proposed D$^2$ algorithm is less sensitive to the data variance across workers. From the theoretical analysis, D-SPIDER-SFO is also robust to that variance. The experiments demonstrate the statement and show that D-SPIDER-SFO converges faster than D$^2$ when the data variance across workers is maximized.

We follow the method proposed in \cite{tang2018d} to create a data distribution with large data variance for the comparison between D-SPIDER-SFO and D$^2$. We conduct the experiments on a server with 5 GPUs and choose the computational complexity as metrics. Each worker only has access to two classes of the whole dataset, called the unshuffled case, and we tune the learning rate of D$^2$ as before.

Figure \ref{Fig:D2}(a) shows that D-PSGD does not converge in the unshuffled case, which is consistent with the original work \cite{tang2018d}. Figure \ref{Fig:D2}(b) shows that D-SPIDER-SFO converges faster than D$^2$, and even it has a similar computing complexity as its centralized implementation. The experiments demonstrate the theoretical statement that D-SPIDER-SFO is robust to the data variance across workers.

\begin{figure*}
    \centering
        \includegraphics[width=1.0\columnwidth]{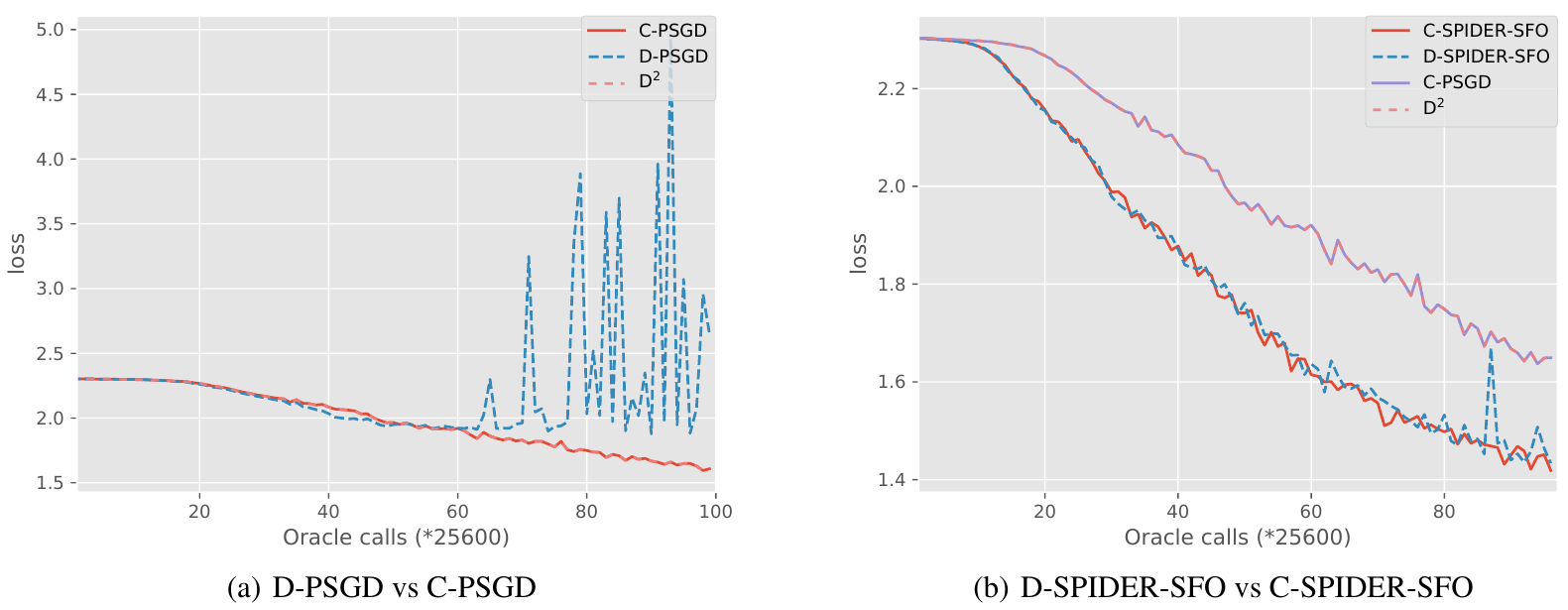}
        \caption{Fig. 2(a) is the comparison between D-PSGD and its centralized parallel version when the data is unshuffled. Fig. 2(b) is the comparison between D-SPIDER-SFO with its centralized counterpart when the data is unshuffled.}
        \label{Fig:D2}
\end{figure*}

\section{Conclusion}

In this paper, we propose D-SPIDER-SFO as a decentralized parallel variant of SPIDER-SFO for a faster convergence rate for nonconvex problems. We theoretically analyze that D-SPIDER-SFO achieves an $\epsilon$-approximate stationary point in the gradient cost of $\mathcal{O}(\epsilon^{-3})$. To the best of our knowledge, D-SPIDER-SFO achieves the state-of-the-art performance for solving nonconvex optimization problems on decentralized networks. Experiments on different network configurations demonstrate the efficiency of the proposed method.

\fontsize{9.3pt}{10.0pt}
\selectfont

\newpage

\allowdisplaybreaks[4]

\setcounter{lemma}{1}
\setcounter{assumption}{0}
\setcounter{theorem}{0}
\setcounter{corollary}{0}

{
{
\centering\LARGE\bf 
D-SPIDER-SFO: A Decentralized Optimization Algorithm with Faster Convergence Rate for Nonconvex Problems

Supplementary Material
\par
}
}
\setcounter{section}{0} 
\renewcommand\thesection{\Alph{section}}
${}$ 

This is the supplementary material of the paper "D-SPIDER-SFO: A Decentralized Optimization Algorithm with Faster Convergence Rate for Nonconvex Problems". We provide the proof to all theoretical results in this paper in this section. To help readers understand the proof, we list the necessary assumptions, which is the same as that in the main submission.

\begin{assumption} 
    We make the following commonly used assumptions for the convergence analysis.
    \begin{enumerate}
        \item \textbf{Lipschitz gradient: } All local loss functions $f_i(\cdot)$ have $L$-Lipschitzian gradients.
        \item \textbf{Average Lipschitz gradient: } In each fixed node $i$, the component function $F_i(x_i;\xi_i)$ has an average L-Lipschitz gradient, that is,
        $$\mathbb{E}\|\nabla F_i(x;\xi_i) - \nabla F_i(y;\xi_i)\|^2\leq L^2\|x - y\|^2,\forall x,y.$$
        \item \textbf{Spectral gap:} Given the symmetric doubly stochastic matrix $W$. Let the eigenvalues of $W\in \mathbb{R}^{n\times n}$ be $\lambda_1\geq \lambda_2\geq \cdots \geq \lambda_n$. We denote by $\lambda$ the second largest value of the set of eigenvalues, i.e., $$\lambda = \max_{i\in \{2,\cdots,n\}} \lambda_i = \lambda_2.$$ We assume $\lambda<1$ and $\lambda_n >-\frac{1}{3}$.
        \item \textbf{Bounded variance: } Assume the variance of stochastic gradient within each worker is bounded, which implies there exists a constant $\sigma$, such that
        $$\mathbb{E}_{\xi\sim \mathcal{D}_i}\|\nabla F_i(x;\xi) - \nabla f_i(x)\|^2\leq \sigma^2, \forall i, \forall x.$$
        \item (For D-PSGD Algorithm only) \textbf{Bounded data variance among workers: } Assume the variance of full gradient among all workers is bounded, which implies that there exists a constant $\zeta$, such that
        $$\mathbb{E}_{i \sim \mathcal{U}([n])}\|\nabla f_i(x) - \nabla f(x)\|^2\leq \zeta^2, \forall i, \forall x. $$
    \end{enumerate}
\end{assumption}
\textbf{Notation}: Let $\|\cdot\|$ be the vector and the matrix $\ell_2$ norm and $\|\cdot\|_F$ be the matrix Frobenius norm. $\nabla f(\cdot)$ denotes the gradient of a function $f$. Let $\mathbf{1}_n$ be the column vector in $\mathbb{R}^n$ with $1$ for all elements and $e_i$ be the column vector with a 1 in the $i$th coordinate and 0's elsewhere. We denote by $f^*$ the optimal solution of $f$. For a matrix $A\in \mathbb{R}^{n\times n}$, let $\lambda_i(A)$ be the $i$-th largest eigenvalue of a matrix. For any fixed integer $j\geq i\geq 0$, let $[i:j]$ be the set $\{i,i+1, \dots, j\}$ and $\{x\}_{i:j}$ be the sequence $\{x_i, x_{i+1}, \dots, x_j\}$.

\subsection{Basics}

Consider the update rule:
\begin{equation}\label{update_rule}
    X_{k+1} = \left\{
    \begin{aligned}
        &[2X_k - X_{k-1} - \eta(\partial F(X_k; \xi_k) - \partial F(X_{k-1};\xi_k))]W\qquad \mod(k,q)\not= 0,\\
        &[2X_k - X_{k-1} - \eta(\partial F(X_k; \xi_k) - G_{k-1})]W \qquad \mod(k,q)= 0.
    \end{aligned}
    \right.
\end{equation}
Since $W$ is symmetric, we have $W = P\diag(\lambda(W)) P^{T}$. Then applying the decomposition to the update rule \eqref{update_rule}, and we have:\\
If $\mod(k,q) \not=0$, then
\begin{align*}
    X_{k+1}P =  & 2X_kP\diag(\lambda(W)) -  X_{k-1} P\diag(\lambda(W))\\
    &-  \eta(\partial F(X_k; \xi_k) - \partial F(X_{k-1};\xi_k))P\diag(\lambda(W)).
\end{align*}
If $\mod(k,q) =0$, then
\begin{align*}
    X_{k+1}P =  & 2X_kP\diag(\lambda(W)) -  X_{k-1} P\diag(\lambda(W))\\
    &-  \eta(\partial F(X_k; \xi_k) - G_{k-1})P\diag(\lambda(W)).
\end{align*}
Let $X_kP = Y_k = (y_{k,1},\dots, y_{k,n})$, $L_k = -\eta(G_{k} - G_{k-1})P = (l_{k,1}, l_{k,2}, \dots, l_{k,n})$, and $V_k = -\frac{1}{\eta}(X_k - X_{k-1}) + (G_{k} - G_{k-1}) = (v_{k,1},\dots, v_{k,n})$. According to the update rule of $G_k$, we have 
\begin{equation}
    L_k = \left\{
    \begin{aligned}
    &-\eta(\partial F(X_k; \xi_k) - \partial F(X_{k-1};\xi_k))P, \qquad \mod(k,q)\not=0,\\
    &-\eta(\partial F(X_k; \xi_k) - G_k)P, \qquad \mod(k,q)=0.
    \end{aligned}
    \right.
\end{equation}
\begin{equation}
     V_k = \left\{
     \begin{aligned}
     &-\frac{1}{\eta}(X_k - X_{k-1}) + (\partial F(X_k;\xi_k) - \partial F(X_{k-1};\xi_{k})),\qquad \mod(k,q)\not=0,\\
     &-\frac{1}{\eta}(X_k - X_{k-1}) + (\partial F(X_k;\xi_k) - G_{k-1}),\qquad \mod(k,q)=0.
     \end{aligned}
    \right.
\end{equation}
Therefore, we have $$y_{k+1,i} = \lambda_i(2y_{k,i} - y_{k-1,i} + l_{k,i}), \forall i\in \{1,2,\dots,n\}.$$ 
Moreover, averaging all local optimization variables, we have
\begin{align}
    \frac{X_{k+1}\mathbf{1}_n}{n} =\frac{X_{k}\mathbf{1}_n}{n} - \frac{\eta V_k\mathbf{1}_n}{n}.
\end{align}

\subsection{Proof of the boundedness of the deviation from the global optimization variable}
\begin{lemma}\label{lemma_add_sequence}
    Given two non-negative sequences $\{a_t\}_{t=1}^\infty$ and $\{b_t\}_{t=1}^\infty$ that satisfying
    $$a_t = \sum_{s=1}^t\rho^{t-s}b_s,$$
    with $\rho\in [0,1)$, we have
    \begin{align*}
        S_k &:=\sum_{t=0}^k a_t\leq \sum_{s=1}^k \frac{b_s}{1-\rho},\\
        D_k &:=\sum_{t=0}^k a^2_t\leq \sum_{s=1}^k \frac{b^2_s}{(1-\rho)^2}.
    \end{align*}
\end{lemma}
\begin{proof}
The proof of this Lemma \ref{lemma_add_sequence} can be found in \cite{tang2018d}.
\end{proof}

\begin{lemma}\label{lemma_add}
	 Given $\rho\in (-\frac{1}{3}, 0)\cup (0,1)$, for any two sequence $\{a_t\}_{t=0}^\infty$ and $\{b_k\}_{t=0}^\infty$ that satisfy 
	\begin{align*}
	    a_0 = b_0 = 0,
	    a_1 = b_1,
	    a_{t+1} = \rho(2a_t - a_{t-1}) + b_t - b_{t-1}, \forall t\geq 1,
	\end{align*}
	we have 
	$$ a_{t+1} = a_1\left(\frac{u^{t+1} - v^{t+1}}{u-v}\right) + \sum_{s=1}^t \beta_s\frac{u^{t-s+1} - v^{t-s+1}}{u-v}, \forall t \geq 0,$$
	where $\beta_s = b_s - b_{s-1}, u = \rho + \sqrt{\rho^2 - \rho}, v = \rho - \sqrt{\rho^2 - \rho}.$
\end{lemma}
\begin{proof}
The proof of this Lemma \ref{lemma_add} can be found in \cite{tang2018d}.
\end{proof}

\begin{lemma}\label{deviation_SPIDER}
Under the Assumption \ref{assumption_SPIDER}, we have 	\begin{align*}
    &\left(1- \frac{48C_2q\eta^2L^2}{S_2}\right)\sum_{k=0}^{K}\sum_{i=1}^n \mathbb{E}  \left\| \frac{X_k\mathbf{1}_n}{n} - x_{k,i}\right\|^2\\
    \leq &2C_1\mathbb{E}\|X_1\|^2_F + \frac{12C_2q\eta^4L^2n}{S_2}\sum_{k=1}^{K-1} \mathbb{E}\left\|\frac{V_{k-1} \mathbf{1}_n}{n} \right\|^2 + \frac{32C_2K\eta^2 \sigma^2}{qS_1} .
\end{align*}
\end{lemma}

\begin{proof}
	Consider the update rule, $$y_{k+1,i} = \lambda_i(2y_{k,i} - y_{k-1,i} + l_{k,i}), \forall i\in \{1,2,\dots,n\}.$$ 
	Applying Lemma \ref{lemma_add}, we have 
	\begin{align}\label{expession_of_yk}
	y_{k+1,i} = \frac{a_i^{k+1} - b_i^{k+1}}{a_i - b_i}y_{1,i} + \lambda_i \sum_{s=1}^{k}\frac{a_i^{k+1-s} - b_i^{k+1-s}}{a_i - b_i} l_{s,i},
	\end{align}
	where we consider $y_{k,i}$, $\lambda_i$, $\lambda_il_{k,i}$, $a_i$, and $b_i$ as $a_t$, $\rho$, $\beta_t$, $u$ and $v$ in Lemma \ref{lemma_add_sequence} respectively.\\
    If $\lambda_{i} \in (-\frac{1}{3}, 0)$, then we have $a_i = \lambda_i +\sqrt{\lambda_i^2 -\lambda_i}\in (0,1)$ and $b_i =\lambda_i -\sqrt{\lambda_i^2 -\lambda_i} \in (-1, 0)$. We have 
	\begin{align}\label{bound_of_bk}
	\|y_{k+1,i}\|^2 \leq &2\|y_{1,i}\|^2 \left|\frac{a_i^{k+1} - b_i^{k+1}}{a_i - b_i}\right|^2 + 2|\lambda_i|^2 \left\|\sum_{s=1}^{k}\frac{a_i^{k+1-s} - b_i^{k+1-s}}{a_i - b_i} l_{s,i}\right\|^2 \notag \\
	\leq &2\|y_{1,i}\|^2 \left|\frac{a_i^{k+1} - b_i^{k+1}}{a_i - b_i}\right|^2 + 2|\lambda_i|^2 \left(\sum_{s=1}^{k}\left|\frac{a_i^{k+1-s} - b_i^{k+1-s}}{a_i - b_i}\right| \|l_{s,i}\|\right)^2.
	\end{align}
	Since $a_i>0$ and $b_i<-a_i<0$, we have
	\begin{align} \label{bound_of_a1}
	    \left|\frac{a_i^{k+1} - b_i^{k+1}}{a_i - b_i}\right| = |b_i^k|\cdot \left|\frac{a_i\left(\frac{a_i}{b_i}\right)^k - b_i}{a_i - b_i}\right| \leq |b_i|^k.
	\end{align}
	Combining \eqref{bound_of_bk} and \eqref{bound_of_a1}, we have 
	\begin{align}\label{bound_of_yk1}
	\|y_{k+1, i} \|^2 \leq 2 \|y_{1,i}\|^2 |b_i|^{2k} + 2|\lambda_i|^2 \left(\sum_{s=1}^{k}|b_i|^{k-s} \|l_{s,i}\|\right)^2.
	\end{align}
	
	If $\lambda_i \in (0,1)$, since $|a_i| = |b_i| = \sqrt{\lambda_i}$, then let $a_i = \sqrt{\lambda_i}e^{i\theta}$ and $b_i = \sqrt{\lambda_i}e^{-i\theta}$. 
	\begin{align}\label{bound_of_sin}
	    \frac{a_i^{k+1} - b_i^{k+1}}{a_i - b_i} = \frac{\sqrt{\lambda_i}^{(k+1)}e^{i(k+1)\theta} - \sqrt{\lambda_i}^{(k+1)}e^{-i(k+1)\theta}}{\sqrt{\lambda_i}e^{i\theta} - \sqrt{\lambda_i}e^{-i\theta}} = \frac{\sqrt{\lambda_i}^{k}\sin((k+1)\theta)}{\sin(\theta)}.
	\end{align}
	Applying \eqref{bound_of_sin} to \eqref{expession_of_yk}, we have, when $k\geq 1$,
	\begin{align}\label{ineq:10}
	    y_{k+1,i} \leq \frac{\sqrt{\lambda_i}^{k}\sin((k+1)\theta)}{\sin(\theta)} y_{1,i} + \lambda_i \sum_{s=1}^{k} \frac{\sqrt{\lambda_i}^{k - s}\sin((k+1-s)\theta)}{\sin(\theta)} l_{s,i}.
	\end{align}
	Clearly, inequality \eqref{ineq:10} holds when $k=0$.
	Then, we have
	\begin{align}\label{bound_of_y2}
	&\|y_{k+1,i}\|^2 |\sin(\theta)|^2 \notag \\
	\leq &2 \left|\sqrt{\lambda_i}^{k}\sin((k+1)\theta)\right|^2 \|y_{1,i}\|^2 + 2|\lambda_i|^2 \left\|\sum_{s=1}^{k} \sqrt{\lambda_i}^{k - s}\sin((k+1-s)\theta) l_{s,i}\right\|^2 \notag \\
	\leq &2 \left|\sqrt{\lambda_i}^{k}\sin((k+1)\theta)\right|^2 \|y_{1,i}\|^2 + 2 |\lambda_i|^2 \left(\sum_{s=1}^{k} \left|\sqrt{\lambda_i}^{k - s}\sin((k+1-s)\theta)\right| \|l_{s,i}\|\right)^2 \notag \\
	\leq &2 \left|\sqrt{\lambda_i}^{k}\right|^2 \|y_{1,i}\|^2 + 2|\lambda_i|^2 \left(\sum_{s=1}^{k} \left|\sqrt{\lambda_i}^{k - s}\right| \|l_{s,i}\|\right)^2.
	\end{align}
	If $\lambda_i \in (-\frac{1}{3}, 0)$, summing \eqref{bound_of_yk1} from $k=0$ to $K-1$, we have
	\begin{align}\label{bound_of_sum_yk1_i}
	\sum_{k=1}^K \|y_{k, i} \|^2
	\leq &\sum_{k=0}^{K-1} \left[2 \|y_{1,i}\|^2 |b_i|^{2k} + 2|\lambda_i|^2 \left(\sum_{s=1}^{k}|b_i|^{k-s} \|l_{s,i}\|\right)^2 \right]\notag \\
	\leq  & 2\|y_{1,i}\| \frac{1}{1-|b_i|^2} + 2|\lambda_i|^2 \sum_{s=1}^{K-1}   \frac{\|l_{s,i}\|^2}{(1-|b_i|)^2},
	\end{align}
	where we use Lemma \ref{lemma_add_sequence} and consider $\sum_{s=1}^{k} |b_i|^{k-s} \|l_{s,i}\|$, $|b_i|^{k-s}$, and $ \|l_{s,i}\|$ as $a_t$, $\rho^{t-s}$, and $b_t$ in Lemma \ref{lemma_add_sequence}.\\
	If $\lambda_i \in (0,1)$, for the similar process, we have
	\begin{align}\label{bound_of_sum_yk2_1}
	&\sum_{k=1}^K  \|y_{k,i}\|^2 |\sin^2(\theta)| \notag \\
	\leq & \sum_{k=0}^{K-1} \left[2 |\sqrt{\lambda_i}^{k}|^2 \|y_{1,i}\|^2 + 2|\lambda_i|^2 \left(\sum_{s=1}^{k} |\sqrt{\lambda_i}^{k - s}| \|l_{s,i}\|\right)^2 \right]\notag \\
	\leq & \frac{2 \|y_{1,i}\|^2}{1 - |\lambda_i|} + 2 |\lambda_i|^2 \sum_{s=1}^{K-1}    \frac{ \|l_{s,i}\|^2}{(1 - \sqrt{\lambda_i})^2},
	\end{align}
	where we use Lemma \ref{lemma_add_sequence} and consider $\sum_{s=1}^{k} |\sqrt{\lambda_i}^{k - s}| \|l_{s,i}\|$, $|\sqrt{\lambda_i}^{k - s}|$, and $ \|l_{s,i}\|$ as $a_t$, $\rho^{t-s}$, and $b_t$.\\
	Since $\sin^2(\theta) =  1 -\lambda_i$ and $\lambda_i\in (0,1)$, we have 
	\begin{align}\label{bound_of_sum_yk2_2}
	\sum_{k=1}^K  \|y_{k,i}\|^2 \leq &\frac{2 \|y_{1,i}\|^2}{(1 - |\lambda_i|)^2} + \frac{2 |\lambda_i|^2}{(1-\lambda_i)(1 - \sqrt{\lambda_i})^2} \sum_{s=1}^{K-1} \|l_{s,i}\|^2 \notag \\
	\leq & \frac{2 \|y_{1,i}\|^2}{(1 - |\lambda|)^2} + \frac{2 |\lambda|^2}{(1-\lambda)(1 - \sqrt{\lambda})^2} \sum_{s=1}^{K-1}  \|l_{s,i}\|^2,
	\end{align}
	where $\lambda = \lambda_2$.\\
	If $\lambda_i \in (-\frac{1}{3},0)$, using \eqref{bound_of_sum_yk1}, then 
	\begin{align}\label{bound_of_sum_yk1}
	\sum_{k=1}^K  \|y_{k,i}\|^2
	\leq & 2\|y_{1,i}\| \frac{1}{1-|b|} + 2|\lambda_n|^2 \sum_{s=1}^{K-1}  \frac{\|l_{s,i}\|^2}{(1-|b|)^2},
	\end{align}
	where $|b| = -\lambda_n + \sqrt{\lambda_n^2 -\lambda_n}$.\\
	Let $C_1 = \max\{\frac{1}{1 - |b|}, \frac{1}{(1-\lambda)^2}\}$ and $C_2 = \max\{\frac{\lambda_n^2}{(1 - |b|^2)}, \frac{\lambda^2}{(1-\sqrt{\lambda})^2 (1-\lambda)}\}$. Therefore, we have 
	\begin{align}\label{final_bound_of_sum_yk}
	    \sum_{k=1}^K  \|y_{k,i}\|^2\leq 2C_1 \|y_{1,i}\|^2 + 2C_2 \sum_{s=1}^{K-1} \|l_{s,i}\|^2.
	\end{align}
    In the next part, we will discuss the term $\sum_{k=0}^{K}\sum_{i=1}^n \mathbb{E}  \left\| \frac{X_k\mathbf{1}_n}{n} - x_{k,i}\right\|^2  $
	\begin{align}\label{relation_of_x_y}
	\sum_{k=0}^{K}\sum_{i=1}^n \mathbb{E}  \left\| \frac{X_k\mathbf{1}_n}{n} - x_{k,i}\right\|^2  =&\sum_{k=0}^{K}\mathbb{E}  \left\| \frac{X_k\mathbf{1}_n}{n}  \mathbf{1}_n^{T} - X_k\right\|^2_F \notag \\
	=&\sum_{k=1}^{K}\mathbb{E} \left\|X_k p_1p_1^T - X_kPP^T\right\|_F^2 \notag \\
	=&\sum_{k=1}^{K}\sum_{i=2}^n \mathbb{E} \|y_{k,i}\|^2.
	\end{align}
	Then, we discuss the term $\sum_{k=1}^K \sum_{i=2}^n \mathbb{E} \|y_{k,i}\|^2$, and firstly, we bound $\sum_{k=1}^{K-1} \sum_{i=2}^n \mathbb{E} \|l_{k,i}\|^2$.\\
	Let $K_0 = \left\lfloor \frac{K}{q}\right\rfloor$, we have
	\begin{align*}
	    \sum_{k=K_0+1}^{K-1} \sum_{i=2}^n \mathbb{E} \|l_{k,i}\|^2 \leq &\eta^2 \sum_{k=K_0+1}^{K-1} \mathbb{E}\left\|\left(\partial F(X_k;\xi_k) - \partial F(X_{k-1};\xi_k)\right)P\right\|^2_F \notag \\
    	\leq & \eta^2 \sum_{k=K_0+1}^{K-1} \sum_{i=1}^n \mathbb{E}\left\|\nabla F_i(x_{k,i};\xi_{k,i}) - \nabla F_i(x_{k-1,i};\xi_{k-1,i})\right\|^2 \notag \\
    	\leq & \eta^2 \sum_{k = K_0+1}^{K-1}\sum_{i=1}^n \frac{L^2}{S_2} \mathbb{E} \|x_{k,i} - x_{k-1,i}\|^2 \notag \\
    	\leq & \frac{\eta^2 L^2}{S_2} \sum_{k = K_0+1}^{K-1} \mathbb{E} \|X_k - X_{k-1}\|_F^2 \notag \\
    	\leq & \frac{\eta^2 L^2}{S_2} \sum_{k = K_0+1}^{K-1} \sum_{i=1}^n \mathbb{E}\|y_{k,i} - y_{k-1,i}\|^2.
	\end{align*}
	Then, we discuss the case that $k\in \{Tq+1,Tq+2,\dots, (T+1)q\}$.
	\begin{align*}
	    \sum_{k = Tq+1}^{(T+1)q}\sum_{i=2}^n\mathbb{E} \|l_{k,i}\|^2\leq &\eta^2 \sum_{k = Tq+1}^{(T+1)q -1} \mathbb{E}\left\|\left(\partial F(X_k;\xi_k) - \partial F(X_{k-1};\xi_k)\right)P\right\|^2_F \notag \\
	    & + \eta^2\mathbb{E}\|\partial F(X_{(T+1)q};\xi_{(T+1)q}) - G_{(T+1)q-1}\|^2.
	\end{align*}
	For convenience, we discuss the term $\mathbb{E}\|\partial F(X_{(T+1)q};\xi_{(T+1)q}) - G_{(T+1)q-1}\|^2$ firstly.
	\begin{align*}
	    &\mathbb{E}\left\|\partial F(X_{(T+1)q};\xi_{(T+1)q}) - G_{(T+1)q-1}\right\|^2\\
	    \leq &\mathbb{E}\left\|\partial F(X_{(T+1)q};\xi_{(T+1)q})- \sum_{k = Tq+1}^{(T+1)q -1}\left[\partial F(X_k;\xi_k) - \partial F(X_{k-1};\xi_k) \right] - \partial F(X_{Tq};\xi_{Tq}) \right\|^2\\
	    \leq & 2\mathbb{E}\left\|\partial F(X_{(T+1)q};\xi_{(T+1)q}) - \partial F(X_{Tq};\xi_{Tq}) \right\|^2 + 2\mathbb{E}\left\|\sum_{k = Tq+1}^{(T+1)q -1}\left[\partial F(X_k;\xi_k) - \partial F(X_{k-1};\xi_k) \right] \right\|^2\\
	    \leq & 2\mathbb{E}\left\|\sum_{k = Tq+1}^{(T+1)q}\left[\partial F(X_{k};\xi_{(T+1)q}) - \partial F(X_{k-1};\xi_{(T+1)q})\right] + \partial F(X_{Tq};\xi_{(T+1)q}) - \partial F(X_{Tq};\xi_{Tq})\right\|^2\\
	    & + 2(q-1)\sum_{k = Tq+1}^{(T+1)q -1}\mathbb{E}\left\|\partial F(X_k;\xi_k) - \partial F(X_{k-1};\xi_k)  \right\|^2\\
	    \leq & 4q\sum_{k = Tq+1}^{(T+1)q}\mathbb{E}\|\partial F(X_{k};\xi_{(T+1)q}) - \partial F(X_{k-1};\xi_{(T+1)q})\|^2 + 8\mathbb{E}\|\partial f(X_{Tq}) - \partial F(X_{Tq};\xi_{(T+1)q}) \|^2\\
	    &+ 8\mathbb{E}\|\partial F(X_{Tq};\xi_{Tq}) - \partial f(X_{Tq})\|^2 + 2(q-1)\sum_{k = Tq+1}^{(T+1)q -1}\mathbb{E}\left\|\partial F(X_k;\xi_k) - \partial F(X_{k-1};\xi_k)  \right\|^2.\\
	    \leq &\frac{(6q-2)L^2}{S_2}\sum_{k = Tq+1}^{(T+1)q}\sum_{i=1}^n\mathbb{E}\|x_{k,i} - x_{k-1,i}\|^2 + \frac{16\sigma^2}{S_1}.
	\end{align*}
	Then, we have 
	\begin{align*}
	    &\sum_{k = Tq+1}^{(T+1)q}\sum_{i=2}^n\mathbb{E} \|l_{k,i}\|^2\\
	    \leq &\eta^2 \sum_{k = Tq+1}^{(T+1)q -1} \mathbb{E}\left\|\left[\partial F(X_k;\xi_k) - \partial F(X_{k-1};\xi_k)\right]P\right\|^2_F  + \eta^2\mathbb{E}\left\|\left[\partial F(X_{(T+1)q};\xi_{(T+1)q}) - G_{(T+1)q-1}\right]P\right\|^2\\
	    \leq & \frac{L^2\eta^2}{S_2} \sum_{k = Tq+1}^{(T+1)q -1}\sum_{i=1}^n \mathbb{E}\left\| x_{k,i} - x_{k-1,i}\right\|^2 + \frac{(6q-2)\eta^2L^2}{S_2}\sum_{k = Tq+1}^{(T+1)q}\sum_{i=1}^n\mathbb{E}\|x_{k,i} - x_{k-1,i}\|^2 + \frac{16\eta^2\sigma^2}{S_1}\\
	    \leq &\frac{6q\eta^2L^2}{S_2}\sum_{k = Tq+1}^{(T+1)q}\sum_{i=1}^n\mathbb{E}\|x_{k,i} - x_{k-1,i}\|^2 + \frac{16\eta^2\sigma^2}{S_1}.
	\end{align*}
	In conclusion, we have
	\begin{align}\label{bound_of_sum_gk}
	\sum_{k=1}^{K-1} \sum_{i=2}^n \mathbb{E} \|l_{k,i}\|^2 \leq &\sum_{T = 0}^{\lfloor \frac{K}{q}\rfloor -1} \sum_{k=Tq+1}^{(T+1)q}\sum_{i=1}^n \mathbb{E} \|g_{k,i}\|^2 + \sum_{k=K_0 +1}^{K-1} \sum_{i=1}^n \mathbb{E} \|g_{k,i}\|^2 \notag\\
	\leq & \frac{6q\eta^2L^2}{S_2}\sum_{k=1}^{K_0} \sum_{i=1}^n\mathbb{E}\|x_{k,i} - x_{k-1,i}\|^2 + \frac{16K\eta^2\sigma^2}{qS_1} \notag\\
	\leq &\frac{6q\eta^2 L^2}{S_2}\sum_{k=1}^{K_0}\mathbb{E}\|X_k - X_{k-1}\|^2_F + \frac{16K\eta^2\sigma^2}{qS_1}\notag\\
	\leq &\frac{6q\eta^2 L^2}{S_2}\sum_{k=1}^{K_0}\sum_{i=1}^n \mathbb{E}\|y_{k,i} - y_{k-1,i}\|^2 + \frac{16K\eta^2\sigma^2}{qS_1}.
	\end{align}
	If $i=1$, we have
	\begin{align}\label{bound_of_differece_y1}
	    \mathbb{E} \|y_{k,1} - y_{k-1,1}\|^2 = n\mathbb{E} \left\|\frac{X_k \mathbf{1}_n}{n} - \frac{X_{k-1} \mathbf{1}_n}{n} \right\|^2 \leq n\eta^2 \mathbb{E}\left \|\frac{V_{k-1} \mathbf{1}_n}{n} \right\|^2.
	\end{align}
	If $i\not =1$, we have 
	\begin{align}\label{sum_of_difference_yk}
	\sum_{k=1}^{K-1}\mathbb{E}\|y_{k,i} - y_{k-1,i}\|^2 \leq &\sum_{k=1}^{K-1}2\mathbb{E}\|y_{k,i}\|^2 + \sum_{k=1}^{K-1}2\mathbb{E}\|y_{k-1,i}\|^2 \notag \\
	\leq &4\sum_{k=1}^{K} \mathbb{E}\|y_{k,i}\|^2.
	\end{align}
    Then, we have 
	\begin{align*}
	&\sum_{k=1}^{K}\sum_{i=2}^n \mathbb{E} \|y_{k,i}\|^2\\ \overset{a}{\leq} &2C_1\sum_{i=2}^n \mathbb{E}\|y_{1,i}\|^2 + 2C_2 \sum_{s=1}^{K-1}\sum_{i=2}^n \mathbb{E}\|g_{s,i}\|^2\\
	\leq & 2C_1\mathbb{E}\|Y_1\|^2_F + 2C_2 \left(\frac{6q\eta^2L^2}{S_2} \sum_{k=1}^{K-1}\sum_{i=1}^n \mathbb{E}\|y_{k,i} - y_{k-1,i}\|^2 + \frac{16K\eta^2 \sigma^2}{qS_1} \right)\\
	\overset{b}{\leq} &  2C_1\mathbb{E}\|Y_1\|^2_F + \frac{12C_2q\eta^4L^2 n}{S_2}\sum_{k=1}^{K-1} \mathbb{E}\left\|\frac{V_{k-1} \mathbf{1}_n}{n} \right\|^2 + \frac{48C_2q\eta^2L^2}{S_2} \sum_{k=1}^{K-1} \sum_{i=2}^n \mathbb{E}\|y_{k,i}\|^2 + \frac{32C_2K\eta^2 \sigma^2}{qS_1}\\
	\overset{c}{\leq} &2C_1\mathbb{E}\|Y_1\|^2_F + \frac{12C_2q\eta^4L^2 n}{S_2}\sum_{k=1}^{K-1} \mathbb{E}\left\|\frac{V_{k-1} \mathbf{1}_n}{n} \right\|^2 \\
	&+ \frac{48C_2q\eta^2L^2}{S_2} \sum_{k=1}^{K-1} \sum_{i=1}^n \mathbb{E}  \left\| \frac{X_k\mathbf{1}_n}{n} - x_{k,i}\right\|^2 + \frac{32C_2K\eta^2 \sigma^2}{qS_1},
	\end{align*}
	where in $\overset{a}{\leq}$, $\overset{b}{\leq}$, and $\overset{c}{\leq}$, we use \eqref{final_bound_of_sum_yk} for $\overset{a}{\leq}$, \eqref{bound_of_differece_y1} and \eqref{sum_of_difference_yk} for $\overset{b}{\leq}$, and \eqref{relation_of_x_y} for $\overset{c}{\leq}$.
	Therefore, using \eqref{relation_of_x_y}, we have 
	\begin{align*}
	    &\left(1- \frac{48C_2q\eta^2L^2}{S_2}\right)\sum_{k=0}^{K}\sum_{i=1}^n \mathbb{E}  \left\| \frac{X_k\mathbf{1}_n}{n} - x_{k,i}\right\|^2\\
	    \leq &2C_1\mathbb{E}\|X_1\|^2_F + \frac{12C_2q\eta^4L^2n}{S_2}\sum_{k=1}^{K-1} \mathbb{E}\left\|\frac{V_{k-1} \mathbf{1}_n}{n} \right\|^2 + \frac{32C_2K\eta^2 \sigma^2}{qS_1}\\
	    \leq & \frac{12C_2q\eta^4L^2n}{S_2}\sum_{k=1}^{K-1} \mathbb{E}\left\|\frac{V_{k-1} \mathbf{1}_n}{n} \right\|^2 + 2\eta^2 \left(\frac{16C_2K}{qS_1} + 3C_1n\right) \sigma^2 + 6C_1n\eta^2\zeta^2 + 6C_1n\eta^2 \|\nabla f(0)\|^2,
	\end{align*}
    where we can expand $\mathbb{E}\|X_1\|^2_F$ by this way.
	\begin{align*}
	    \mathbb{E}\|X_1\|^2_F &= \mathbb{E}\|(X_0 - \eta \partial F(X_0;\xi_0))W\|^2_F\\
	    & = \mathbb{E}\|X_0 - \eta \partial F(X_0;\xi_0)\|^2_F\\
	    & = \eta^2  \mathbb{E}\| \partial F(0;\xi_0)\|^2_F\\
	    & = \eta^2 \sum_{i=1}^n \mathbb{E}\| (\nabla F_i(0;\xi_0) - \nabla f_i(0)) +(\nabla f_i(0) - \nabla f(0)) + \nabla f(0)\|^2_F\\
	    & = 3n\eta^2 (\sigma^2 + \zeta^2 + \|\nabla f(0)\|^2)
	\end{align*}
\end{proof}

\setcounter{lemma}{0}

\begin{lemma}\label{lemma2}
	Under the Assumption \ref{assumption_SPIDER}, we have 
	\begin{align*}
	    &\frac{1}{K}\sum_{k=0}^{K-1} \mathbb{E} \left\|\frac{\partial f(X_k) \mathbf{1}_n}{n} - \frac{V_k \mathbf{1}_n}{n} \right\|^2\\
	    \leq& \frac{12C_1L^2 q}{KnDS_2}\mathbb{E}\|X_1\|^2_F + \left(\frac{72C_2\eta^4L^4q^2}{KDS^2_2} + \frac{3qL^2\eta^2}{KS_2}\right)\sum_{k=1}^{K-1} \mathbb{E}\left\|\frac{V_{k-1} \mathbf{1}_n}{n} \right\|^2 + \left(1+ \frac{192C_2L^2\eta^2}{nDS_1}\right)\frac{\sigma^2}{S_1}.
	\end{align*}
\end{lemma}

\begin{proof}

	Consider the term $\mathbb{E}\left\|\frac{\partial f(X_k) \mathbf{1}_n}{n} - \frac{V_k \mathbf{1}_n}{n} \right\|^2$.
	\begin{align*}
	&\mathbb{E} \left\|\frac{\partial f(X_k) \mathbf{1}_n}{n} - \frac{V_k \mathbf{1}_n}{n} \right\|^2 \\
	= & \mathbb{E} \left\|\frac{\partial f(X_k) \mathbf{1}_n}{n} - \frac{(-\frac{1}{\eta}(X_k - X_{k-1}) + (\partial F(X_k;\xi_k) - \partial F(X_{k-1};\xi_{k})))\mathbf{1}_n}{n}\right\|^2\\
	= & \mathbb{E} \left\|\frac{\partial f(X_k) \mathbf{1}_n}{n} - \frac{(\partial F(X_k;\xi_k) - \partial F(X_{k-1};\xi_{k}))\mathbf{1}_n}{n} +   \frac{\frac{1}{\eta}(X_k - X_{k-1})\mathbf{1}_n}{n}\right\|^2\\
	\leq & \mathbb{E}\left\|\frac{[\partial f(X_k) - \partial f(X_{k-1})] \mathbf{1}_n}{n} - \frac{[\partial F(X_k;\xi_k) - \partial F(X_{k-1};\xi_{k})]\mathbf{1}_n}{n} +   \frac{[\partial f(X_{k-1}) - V_{k-1}]\mathbf{1}_n}{n}\right\|^2\\
	\leq & \mathbb{E}\left\|\left\{[\partial f(X_k) - \partial f(X_{k-1})]  - [\partial F(X_k;\xi_k) - \partial F(X_{k-1};\xi_{k})]\right\}\frac{\mathbf{1}_n}{n}\right\|^2 \\
	&+ \mathbb{E}\left\|\frac{[\partial f(X_{k-1}) - V_{k-1}]\mathbf{1}_n}{n}\right\|^2\\
	\leq & \sum_{j=k_0+1}^{k}\mathbb{E}\left\|[\partial F(X_j;\xi_j) - \partial F(X_{j-1};\xi_{j})]\frac{\mathbf{1}_n}{n}\right\|^2 + \mathbb{E} \left\|\frac{\partial f(X_{k_0}) \mathbf{1}_n}{n} - \frac{V_{k_0} \mathbf{1}_n}{n} \right\|^2\\
	\leq & \frac{L^2}{nS_2}\sum_{j=k_0+1}^{k}\sum_{i=1}^{n}\mathbb{E}\left\|x_{j,i} - x_{j-1,i}\right\|^2 + \mathbb{E} \left\|\frac{\partial f(X_{k_0}) \mathbf{1}_n}{n} - \frac{V_{k_0} \mathbf{1}_n}{n} \right\|^2\\
	\leq &\sum_{j=k_0+1}^{k} \frac{L^2}{nS_2} \sum_{i=1}^n \mathbb{E}\left\|x_{j,i} - \frac{X_j\mathbf{1}_n}{n}- \left(x_{j-1,i} - \frac{X_{j-1,i}\mathbf{1}_n}{n} \right) + \left(\frac{X_k\mathbf{1}_n}{n} - \frac{X_{k-1}\mathbf{1}_n}{n}\right)\right\|^2\\
	&+ \mathbb{E}\left\|\frac{[\partial f(x_{k-1}) - V_{k-1}]\mathbf{1}_n}{n}\right\|^2\\
	\leq &\sum_{j=k_0+1}^{k} \frac{3L^2}{nS_2} \sum_{i=1}^n \left[\mathbb{E} \left\| \frac{X_j \mathbf{1}_n}{n} - x_{j,i} \right\|^2 + \mathbb{E} \left\| \frac{X_{j-1} \mathbf{1}_n}{n} - x_{j-1,i} \right\|^2 + \mathbb{E} \left\| \frac{X_j \mathbf{1}_n}{n} - \frac{X_{j-1} \mathbf{1}_n}{n} \right\|^2 \right] \\
	&+ \mathbb{E} \left\|\frac{\partial f(X_{k_0}) \mathbf{1}_n}{n} - \frac{V_{k_0} \mathbf{1}_n}{n} \right\|^2.
	\end{align*}
	Summing from $k=k_0$ to $k = K-1$, we have
	\begin{align*}
	&\sum_{k=k_0}^{K-1}\mathbb{E} \left\|\frac{\partial f(X_k) \mathbf{1}_n}{n} - \frac{V_k \mathbf{1}_n}{n} \right\|^2\\
	\leq & \sum_{k=k_0+1}^{K-1} \sum_{j=k_0+1}^{k} \frac{3L^2}{nS_2} \sum_{i=1}^n \left[\mathbb{E} \left\|\frac{X_j \mathbf{1}_n}{n} - x_{j,i} \right\|^2 + \mathbb{E} \left\| \frac{X_{j-1} \mathbf{1}_n}{n} - x_{j-1,i} \right\|^2 + \mathbb{E} \left\| \frac{X_j \mathbf{1}_n}{n} - \frac{X_{j-1} \mathbf{1}_n}{n} \right\|^2 \right] \\
	&+ \sum_{k=k_0}^{K-1} \mathbb{E} \left\|\frac{\partial f(X_{k_0}) \mathbf{1}_n}{n} - \frac{V_{k_0} \mathbf{1}_n}{n} \right\|^2\\
	\leq & \frac{6L^2q}{nS_2}  \sum_{k=k_0}^{K-1} \sum_{i=1}^n \mathbb{E} \left\| \frac{X_k \mathbf{1}_n}{n} - x_{k,i} \right\|^2 + \frac{3qL^2\eta^2}{S_2} \sum_{k=k_0}^{K-1} \mathbb{E} \left\|\frac{V_k \mathbf{1}_n}{n} \right\|^2 \\
	&+ (K-k_0) \mathbb{E} \left\|\frac{\partial f(X_{k_0}) \mathbf{1}_n}{n} - \frac{V_{k_0} \mathbf{1}_n}{n} \right\|^2.
	\end{align*}
	Consider the term $\mathbb{E} \left\|\frac{\partial f(X_{k_0}) \mathbf{1}_n}{n} - \frac{V_{k_0} \mathbf{1}_n}{n} \right\|^2$,
    \begin{align*}
        &\mathbb{E} \left\|\frac{\partial f(X_{k_0}) \mathbf{1}_n}{n} - \frac{V_{k_0} \mathbf{1}_n}{n} \right\|^2\\
        =&\mathbb{E} \left\|\frac{\partial f(X_{k_0}) \mathbf{1}_n}{n} - \left[\frac{X_{k_0} - X_{k_0-1}}{-\eta} + \partial F(X_{k_0};\xi_{k_0}) - G_{k_0-1} \right]\frac{\mathbf{1}_n}{n} \right\|^2\\
        =&\mathbb{E}\left\|\left[\frac{\partial f(X_{k_0}) \mathbf{1}_n}{n} - \frac{\partial F(X_{k_0};\xi_{k_0})\mathbf{1}_n}{n}\right] - \left[\frac{X_{k_0} - X_{k_0-1}}{-\eta}  - G_{k_0-1} \right]\frac{\mathbf{1}_n}{n} \right\|^2\\
        =&\mathbb{E}\left\|\left[\frac{\partial f(X_{k_0}) \mathbf{1}_n}{n} - \frac{\partial F(X_{k_0};\xi_{k_0})\mathbf{1}_n}{n}\right] - \left[\frac{X_{k_0-1} - X_{k_0-2}}{-\eta} + G_{k_0-1} - G_{k_0-2}  - G_{k_0-1} \right]\frac{\mathbf{1}_n}{n} \right\|^2\\
        =&\mathbb{E}\left\|\left[\frac{\partial f(X_{k_0}) \mathbf{1}_n}{n} - \frac{\partial F(X_{k_0};\xi_{k_0})\mathbf{1}_n}{n}\right] - \left[\frac{X_{k_0-1} - X_{k_0-2}}{-\eta} - G_{k_0-2} \right]\frac{\mathbf{1}_n}{n} \right\|^2\\
        =&\mathbb{E}\left\|\left[\frac{\partial f(X_{k_0}) \mathbf{1}_n}{n} - \frac{\partial F(X_{k_0};\xi_{k_0})\mathbf{1}_n}{n}\right] - \left[\frac{X_{1} - X_{0}}{-\eta} - G_0\right]\frac{\mathbf{1}_n}{n} \right\|^2\\
        =&\mathbb{E}\left\|\left[\frac{\partial f(X_{k_0}) \mathbf{1}_n}{n} - \frac{\partial F(X_{k_0};\xi_{k_0})\mathbf{1}_n}{n}\right] \right\|^2\\
        \leq &\frac{1}{n}\sum_{i=1}^n \mathbb{E}\left\|\nabla f_i(x_{k_0,i}) - \nabla F_i(x_{k_0,i};\xi_{k_0,i}) \right\|^2 \\
        \leq &\frac{\sigma^2}{S_1}.
    \end{align*}	
	Therefore, we have 
	\begin{align*}
	&\sum_{k=0}^{K-1} \mathbb{E} \left\|\frac{\partial f(X_k) \mathbf{1}_n}{n} - \frac{V_k \mathbf{1}_n}{n} \right\|^2\\
	\overset{d}{\leq} & \frac{6L^2q}{nS_2}  \sum_{k=0}^{K-1} \sum_{i=1}^n \mathbb{E} \left\| \frac{X_k \mathbf{1}_n}{n} - x_{k,i} \right\|^2 + \frac{3qL^2\eta^2}{S_2} \sum_{k=0}^{K-1} \mathbb{E} \left\|\frac{V_k \mathbf{1}_n}{n}\right\|^2 + K \frac{\sigma^2}{S_1},
	\end{align*}
	where in $\overset{d}{\leq}$, we use $\mathbb{E}\left\|\frac{\partial f(X_{k_0}) \mathbf{1}_n}{n} - \frac{V_{k_0} \mathbf{1}_n}{n} \right\|^2\leq \frac{\sigma^2}{S_1}$.\\
	Applying Lemma \ref{lemma_add_sequence}, we have
	\begin{align*}
	    &\sum_{k=0}^{K-1} \mathbb{E} \left\|\frac{\partial f(X_k) \mathbf{1}_n}{n} - \frac{V_k \mathbf{1}_n}{n} \right\|^2\\
	    \leq & \frac{6L^2q}{nS_2}\left(\frac{2C_1}{D}\mathbb{E}\|X_1\|^2_F + \frac{12C_2q\eta^4L^2n}{DS_2}\sum_{k=1}^{K-1} \mathbb{E}\left\|\frac{V_{k-1} \mathbf{1}_n}{n} \right\|^2 + \frac{32C_2K\eta^2 \sigma^2}{DqS_1}\right) + \frac{3qL^2\eta^2}{S_2} \sum_{k=0}^{K-1} \mathbb{E} \left\|\frac{V_k \mathbf{1}_n}{n}\right\|^2 + K \frac{\sigma^2}{S_1}\\
	    \leq& \frac{12C_1L^2 q}{nDS_2}\mathbb{E}\|X_1\|^2_F + \left(\frac{72C_2\eta^4L^4q^2}{DS^2_2} + \frac{3qL^2\eta^2}{S_2}\right)\sum_{k=1}^{K-1} \mathbb{E}\left\|\frac{V_{k-1} \mathbf{1}_n}{n} \right\|^2 +\left(1+ \frac{192C_2L^2\eta^2}{nDS_1}\right)\frac{K\sigma^2}{S_1}.
	\end{align*}
	Therefore, we have
	\begin{align*}
	    &\frac{1}{K}\sum_{k=0}^{K-1} \mathbb{E} \left\|\frac{\partial f(X_k) \mathbf{1}_n}{n} - \frac{V_k \mathbf{1}_n}{n} \right\|^2\\
	    \leq& \frac{12C_1L^2 q}{KnDS_2}\mathbb{E}\|X_1\|^2_F + \left(\frac{72C_2\eta^4L^4q^2}{KDS^2_2} + \frac{3qL^2\eta^2}{KS_2}\right)\sum_{k=1}^{K-1} \mathbb{E}\left\|\frac{V_{k-1} \mathbf{1}_n}{n} \right\|^2 + \left(1+ \frac{192C_2L^2\eta^2}{nDS_1}\right)\frac{\sigma^2}{S_1}.
	\end{align*}
\end{proof}

\begin{theorem}\label{theorem3}
	For the on-line case, set the parameters $S_1$, $S_2$, $\eta$ and $q$. Then under the Assumption \ref{assumption_SPIDER}, for Algorithm DCSPIDER-SFO, we have 
	\begin{align*}
	&\frac{1}{K}\sum_{k=1}^K \mathbb{E}\left\|\nabla f\left(\frac{X_k \mathbf{1}_n}{n}\right)\right\|^2 + \frac{M}{K}\sum_{k=0}^{K-1}\mathbb{E}\left\| \frac{V_k \mathbf{1}_n}{n}\right\|^2\\
	\leq & \frac{2\mathbb{E}[f(\frac{X_0 \mathbf{1}_n}{n}) - f^*]}{\eta K}
	+  \left(1 + \frac{32C_2 L^2 \eta^2}{qS_2 D} + \frac{192C_2 L^2 \eta^2}{n S_1 S_2 D}\right)\frac{2\sigma^2}{S_1} \\
	&+ \frac{3\eta^2}{K}\left(\frac{4L^2C_1}{D} +  \frac{24L^2C_1 q}{D S_2}\right) (\sigma^2 + \zeta^2 + \|\nabla f(0)\|^2).
	\end{align*}
	where $D, C_2$ are defined in Lemma \ref{deviation_SPIDER} and $M:=\left(1 - L\eta -\frac{6qL^2\eta^2}{S_2} -  \frac{144 C_2 L^4 \eta^4 q^2}{D S_2^2} - \frac{24C_2 q\eta^4 L^4}{S_2 D}\right)$.
\end{theorem}

\begin{proof}
\begin{align}\label{ineq:f_k}
    \mathbb{E} f\left(\frac{X_{k+1}\mathbf{1}_n}{n}\right) = &\mathbb{E} f\left(\frac{\left(X_{k}W - \eta V_k\right)\mathbf{1}_n}{n}\right) \notag \\
    \leq & \mathbb{E} f\left(\frac{X_{k}\mathbf{1}_n}{n}\right) - \eta \mathbb{E}\left\langle \nabla f\left(\frac{X_{k}\mathbf{1}_n}{n}\right), \frac{V_k \mathbf{1}_n}{n}\right\rangle +\frac{L\eta^2}{2} \mathbb{E}\left\|\frac{V_k \mathbf{1}_n}{n}\right\|^2 \notag \\
    = &\mathbb{E}f\left(\frac{X_{k}\mathbf{1}_n}{n}\right) - \frac{\eta}{2} \left(\mathbb{E}\left\|\nabla f\left(\frac{X_{k}\mathbf{1}_n}{n}\right)\right\|^2 + \mathbb{E}\left\| \frac{V_k \mathbf{1}_n}{n}\right\|^2 \right) \notag \\
    &+ \frac{\eta}{2}  \mathbb{E}\left\|\nabla f\left(\frac{X_{k}\mathbf{1}_n}{n}\right) - \frac{V_k\mathbf{1}_n}{n} \right\|^2  + \frac{L\eta^2}{2} \mathbb{E}\left\|\frac{V_k \mathbf{1}_n}{n}\right\|^2 \notag \\
	=&\mathbb{E}f\left(\frac{X_{k}\mathbf{1}_n}{n}\right) - \frac{\eta}{2}\mathbb{E}\left\|\nabla f\left(\frac{X_{k}\mathbf{1}_n}{n}\right)\right\|^2 + \frac{\eta\left(L\eta -1\right)}{2} \mathbb{E}\left\| \frac{V_k \mathbf{1}_n}{n}\right\|^2 \notag \\
	& + \frac{\eta}{2}\mathbb{E} \left\|\nabla f\left(\frac{X_{k}\mathbf{1}_n}{n}\right) - \frac{V_k \mathbf{1}_n}{n} \right\|^2.
\end{align}
	
We discuss the terms $\mathbb{E} \left\|\nabla f\left(\frac{X_{k}\mathbf{1}_n}{n}\right) - \frac{\partial f\left(X_k\right) \mathbf{1}_n}{n} \right\|^2$ and $\mathbb{E}\left\|\frac{V_k \mathbf{1}_n}{n} - \frac{\partial f\left(X_k\right) \mathbf{1}_n}{n} \right\|^2$.
    \begin{align}
	&\sum_{k=0}^{K-1} \mathbb{E} \left\|\nabla f\left(\frac{X_{k}\mathbf{1}_n}{n}\right) - \frac{\partial f\left(X_k\right) \mathbf{1}_n}{n} \right\|^2 \notag \\
	\leq & \sum_{k=0}^{K-1}\frac{1}{n}\sum_{i=1}^n \mathbb{E}\left\|\nabla f_i\left(\frac{X_{k}\mathbf{1}_n}{n}\right) - f_i\left(x_{k,i}\right)\right\|^2 \notag \\
	\leq & \sum_{k=0}^{K-1} \frac{L^2}{n}\sum_{i=1}^n \mathbb{E}\left\|\frac{X_{k}\mathbf{1}_n}{n} - x_{k,i}\right\|^2 \notag \\
	\overset{e}{\leq} & \frac{2L^2C_1\mathbb{E}\|X_1\|^2_F}{Dn} + \frac{12C_2q\eta^4 L^4 n}{S_2 Dn} \sum_{k=0}^{K-2} \mathbb{E} \left\|\frac{V_{k} \mathbf{1}_n}{n} \right\|^2 + \frac{32C_2 K\eta^2 \sigma^2 L^2}{qS_1Dn}\notag \\
	\leq & \frac{2L^2C_1\mathbb{E}\|X_1\|^2_F}{Dn} + \frac{12C_2q\eta^4 L^4 }{S_2 D} \sum_{k=0}^{K-1} \mathbb{E} \left\|\frac{V_{k} \mathbf{1}_n}{n} \right\|^2 + \frac{32C_2 K\eta^2 \sigma^2 L^2}{qS_1Dn},
	\end{align}
	where in $\overset{e}{\leq}$, we use Lemma \eqref{deviation_SPIDER} and $D = 1 - \frac{48C_2 q\eta^2L^2}{S_2}$.\\
	Applying Lemma \eqref{deviation_SPIDER} and Lemma \eqref{lemma2}, we have
	\begin{align}\label{ineq:variance_of_derivative}
	&\sum_{k=0}^{K-1} \mathbb{E}\left\|\nabla f\left(\frac{X_{k}\mathbf{1}_n}{n}\right) - \frac{V_k \mathbf{1}_n}{n} \right\|^2 \notag \\
	\leq &\sum_{k=0}^{K-1} 2 \mathbb{E} \left\|\nabla f\left(\frac{X_{k}\mathbf{1}_n}{n}\right) - \frac{\partial f\left(X_k\right) \mathbf{1}_n}{n} \right\|^2  + \sum_{k=1}^{K-1} 2\mathbb{E} \left\|\frac{\partial f(X_k) \mathbf{1}_n}{n} - \frac{V_k \mathbf{1}_n}{n} \right\|^2 \notag \\
	\leq &2\left( \frac{2L^2C_1\mathbb{E}\|X_1\|^2_F}{Dn} + \frac{12C_2q\eta^4 L^4 }{S_2 D} \sum_{k=0}^{K-1} \mathbb{E} \left\|\frac{V_{k} \mathbf{1}_n}{n} \right\|^2 + \frac{32C_2 K\eta^2 \sigma^2 L^2}{qS_1Dn}\right) \notag \\
	& + 2\left(\frac{6L^2q}{nS_2}  \sum_{k=1}^{K-1} \sum_{i=1}^n \mathbb{E} \left\| \frac{X_k \mathbf{1}_n}{n} - x_{k,i} \right\|^2 + \frac{3qL^2\eta^2}{S_2} \sum_{k=0}^{K-1} \mathbb{E} \left\|\frac{V_k \mathbf{1}_n}{n} \right\|^2 + K \frac{\sigma^2}{S_1} \right) \notag \\
	\overset{f}{\leq} &\left(\frac{4L^2C_1}{Dn} +  \frac{24L^2C_1 q}{D S_2n} \right) \mathbb{E}\|X_1\|^2_F + \left(\frac{6qL^2\eta^2}{S_2} +  \frac{144 C_2 L^4 \eta^4 q^2}{D S_2^2} + \frac{24C_2 q\eta^4 L^4}{S_2 D}\right)\sum_{k=0}^{K-1} \mathbb{E} \left\|\frac{V_k \mathbf{1}_n}{n} \right\|^2 \notag \\
	& + 2\left(1 + \frac{32C_2 L^2 \eta^2}{nq D} + \frac{192C_2 L^2 \eta^2}{n S_2 D}\right) \frac{K\sigma^2}{S_1},
	\end{align}
	where if $\overset{e}{\leq}$, we use \eqref{deviation_SPIDER} and $D = 1 - \frac{48C_2 q\eta^2L^2}{S_2}$.\\
    Summing \eqref{ineq:f_k} from $k=0$ to $k=K-1$, we have
    \begin{align*}
    & \sum_{k=0}^{K-1} \frac{\eta}{2}\mathbb{E}\left\|\nabla f\left(\frac{X_{k}\mathbf{1}_n}{n}\right)\right\|^2 + \sum_{k=0}^{K-1}\frac{\eta\left(1 - L\eta\right)}{2} \mathbb{E}\left\| \frac{V_k \mathbf{1}_n}{n}\right\|^2\\
    \leq & \sum_{k=0}^{K-1} \mathbb{E} \left[f\left(\frac{X_{k}\mathbf{1}_n}{n}\right) - f\left(\frac{X_{k+1}\mathbf{1}_n}{n}\right) \right] + \sum_{k=0}^{K-1}\frac{\eta}{2}\mathbb{E} \left\|\nabla f\left(\frac{X_{k}\mathbf{1}_n}{n}\right) - \frac{V_k \mathbf{1}_n}{n} \right\|^2.
    \end{align*}
    Using \eqref{ineq:variance_of_derivative}, it implies that
    \begin{align*}
     & \frac{\eta}{2} \sum_{k=0}^{K-1} \mathbb{E}\left\|\nabla f\left(\frac{X_{k}\mathbf{1}_n}{n}\right)\right\|^2 + \frac{\eta}{2}\left(1 - L\eta -\frac{6qL^2\eta^2}{S_2} -  \frac{144 C_2 L^4 \eta^4 q^2}{D S_2^2} - \frac{24C_2 q\eta^4 L^4}{S_2 D}\right) \sum_{k=0}^{K-1}\mathbb{E}\left\| \frac{V_k \mathbf{1}_n}{n}\right\|^2\\
     \leq & \mathbb{E}\left[f\left(\frac{X_{0}\mathbf{1}_n}{n}\right) - f^* \right] + \frac{\eta}{2}\left(\frac{4L^2C_1}{Dn} +  \frac{24L^2C_1 q}{D S_2n}\right) \mathbb{E}\|X_1\|^2_F + \frac{\eta}{2} \left(1 + \frac{32C_2 L^2 \eta^2}{nq D} + \frac{192C_2 L^2 \eta^2}{n S_2 D}\right)\frac{2K\sigma^2}{S_1}.
    \end{align*}
    Since $\mathbb{E}\|X_1\|^2_F \leq 3n\eta^2 (\sigma^2 + \zeta^2 + \|\nabla f(0)\|^2)$, we have 
    \begin{align*}
        & \frac{\eta}{2} \sum_{k=0}^{K-1} \mathbb{E}\left\|\nabla f\left(\frac{X_{k}\mathbf{1}_n}{n}\right)\right\|^2 + \frac{\eta}{2}\left(1 - L\eta -\frac{6qL^2\eta^2}{S_2} -  \frac{144 C_2 L^4 \eta^4 q^2}{D S_2^2} - \frac{24C_2 q\eta^4 L^4}{S_2 D}\right) \sum_{k=0}^{K-1}\mathbb{E}\left\| \frac{V_k \mathbf{1}_n}{n}\right\|^2\\
        \leq & \mathbb{E}\left[f\left(\frac{X_{0}\mathbf{1}_n}{n}\right) - f^* \right]  + \frac{\eta}{2} \left(1 + \frac{32C_2 L^2 \eta^2}{nq D} + \frac{192C_2 L^2 \eta^2}{n S_2 D}\right)\frac{2K\sigma^2}{S_1} \\
        &+ \frac{3 n\eta^3}{2}\left(\frac{4L^2C_1}{Dn} +  \frac{24L^2C_1 q}{D S_2n}\right) (\sigma^2 + \zeta^2 + \|\nabla f(0)\|^2).
    \end{align*}
    Let $M =\left(1 - L\eta -\frac{6qL^2\eta^2}{S_2} -  \frac{144 C_2 L^4 \eta^4 q^2}{D S_2^2} - \frac{24C_2 q\eta^4 L^4}{S_2 D}\right) $. We have
    \begin{align*}
	&\frac{1}{K}\sum_{k=1}^K \mathbb{E}\left\|\nabla f\left(\frac{X_k \mathbf{1}_n}{n}\right)\right\|^2 + \frac{M}{K}\sum_{k=0}^{K-1}\mathbb{E}\left\| \frac{V_k \mathbf{1}_n}{n}\right\|^2\\
	\leq & \frac{2\mathbb{E}[f(\frac{X_0 \mathbf{1}_n}{n}) - f^*]}{\eta K}
	+  \left(1 + \frac{32C_2 L^2 \eta^2}{nq D} + \frac{192C_2 L^2 \eta^2}{n  S_2 D}\right)\frac{2\sigma^2}{S_1} \\
	&+ \frac{3\eta^2}{K}\left(\frac{4L^2C_1}{D} +  \frac{24L^2C_1 q}{D S_2}\right) (\sigma^2 + \zeta^2 + \|\nabla f(0)\|^2).
	\end{align*}
\end{proof}

\subsection{proof of Corollary \ref{corollary: convergence rate}}
\begin{corollary}\label{corollary: convergence rate_appendix}
    Set the parameters $S_1 = \frac{\sigma^2}{\epsilon^2}, S_2 = \frac{\sigma}{\epsilon}, q = \frac{\sigma}{\epsilon}, \eta<\min\left(\frac{-1+\sqrt{13}}{12 L}, \frac{1}{4\sqrt{6C_2}L}\right)$ and $K = \lfloor \frac{l}{\epsilon^2}\rfloor + 1$. 
 Then under the Assumption  \ref{assumption_SPIDER}, running Algorithm D-SPIDER-SFO for $K$ iterations, we have
 $$ \frac{1}{K}\sum_{k=0}^{K-1}\mathbb{E} \left\|\nabla f\left(\frac{X_k\mathbf{1}}{n} \right)\right\|^2\leq 3\epsilon^2 + \frac{448 C_2L^2\eta^2\epsilon^3}{nD\sigma},$$
 where 
 \begin{align*}
     l &:= \frac{2\mathbb{E}[f(\frac{X_0\mathbf{1}_n}{n}) -f^*]}{\eta} + \frac{84C_1L^2\eta^2}{D}(\sigma^2 + \xi^2 + \|\nabla f(0)\|^2),
 \end{align*}
 The gradient cost is bounded by $2l\sigma\epsilon^{-3} + 2\sigma^2\epsilon^{-2}$.
\end{corollary}

\begin{proof}

Since $S_1 = \frac{\sigma^2}{\epsilon^2}, S_2 = \frac{\sigma}{\epsilon}, q = \frac{\sigma}{\epsilon}$ and $ \eta<\min\left(\frac{-1+\sqrt{13}}{12 L}, \frac{1}{8\sqrt{3C_2}L}\right) $, we have  
$$1-L\eta - 6qL^2\eta^2 >\frac{1}{2}$$ 
and $$\frac{1}{2} - 48C_2L^2\eta^2 + 48C_2 L^3\eta^3>0,$$ that is, $$1-L\eta -\frac{6qL^2\eta^2}{S_2}- 48C_2L^2\eta^2 + 48C_2 L^3\eta^3>0.$$
Since $1 - L\eta -\frac{6qL^2\eta^2}{S_2} -  \frac{144 C_2 L^4 \eta^4 q^2}{D S_2^2} - \frac{24C_2 q\eta^4 L^4}{S_2 D}>0$ equals to $1-L\eta -\frac{6qL^2\eta^2}{S_2}- 48C_2L^2\eta^2 + 48C_2 L^3\eta^3+120C_2L^4\eta^4>0$, we have $M>0$.
Therefore, we have 
\begin{align*}
	&\frac{1}{K}\sum_{k=1}^K \mathbb{E}\left\|\nabla f\left(\frac{X_k \mathbf{1}_n}{n}\right)\right\|^2 \\
	\leq & \frac{2\mathbb{E}[f(\frac{X_0 \mathbf{1}_n}{n}) - f^*]}{\eta K}
	+  \left(1 + \frac{32C_2 L^2 \eta^2}{n q D} + \frac{192C_2 L^2 \eta^2}{n S_2 D}\right)\frac{2\sigma^2}{S_1} \\
	&+ \frac{3\eta^2}{K}\left(\frac{4L^2C_1}{D} +  \frac{24L^2C_1 q}{D S_2}\right) (\sigma^2 + \zeta^2 + \|\nabla f(0)\|^2).
\end{align*}
Let $l =\left(3\eta^2\left(\frac{4L^2C_1}{D} +  \frac{24L^2C_1 q}{D S_2}\right) (\sigma^2 + \zeta^2 + \|\nabla f(0)\|^2) + \frac{2\mathbb{E}[f(\frac{X_0 \mathbf{1}_n}{n}) - f^*]}{\eta}\right) $ and $K = \left\lfloor \frac{l}{\epsilon^2}\right\rfloor +1$. We have
$$\frac{1}{K}\sum_{k=0}^{K-1}\mathbb{E} \left\|\nabla f\left(\frac{X_k\mathbf{1}}{n} \right)\right\|^2\leq 3\epsilon^2 + \frac{448 C_2L^2\eta^2\epsilon^3}{nD\sigma},$$
 where 
 \begin{align*}
     l &:= \frac{2\mathbb{E}[f(\frac{X_0\mathbf{1}_n}{n}) -f^*]}{\eta} + \frac{84C_1\eta^2L^2(\sigma^2 + \zeta^2 + \|\nabla f(0)\|^2)}{D}.
 \end{align*}
Finally, we compute the gradient cost for finding an $\epsilon$-approximated first-order stationary point.
\begin{align*}
    \left(\left\lfloor \frac{K}{q} \right\rfloor + 1\right)\cdot((q-1)S_2 + S_1)\leq &\frac{K}{q}(q-1)S_2 + \frac{KS_1}{q} + (q-1)S_2 + S_1\\
    \leq & KS_2 - \frac{KS_2}{q} + \frac{KS_1}{q} + qS_2 -S_2 + S_1\\
    \leq & \left(\frac{l}{\epsilon^2} +1 \right)\frac{\sigma}{\epsilon} - \frac{l}{\epsilon^2} + \frac{l\sigma}{\epsilon^3} + \frac{\sigma}{\epsilon} + \frac{\sigma^2}{\epsilon^2} - \frac{\sigma}{\epsilon} + \frac{\sigma^2}{\epsilon^2}\\
    \leq \frac{2l\sigma}{\epsilon^3}+ \frac{2\sigma^2}{\epsilon^2} + \frac{\sigma}{\epsilon} -\frac{l}{\epsilon^2}& ,
\end{align*}
i.e., the gradient cost is bounded by $\mathcal{O}(\frac{2l\sigma}{\epsilon^3}+ \frac{2\sigma^2}{\epsilon^2} + \frac{\sigma}{\epsilon})$
We complete the proof of the computation complexity of D-SPIDER-SFO.\\
\end{proof}

\newpage
\subsection{Experiments}
\begin{figure*}[ht]
	\centering  
		\includegraphics[width=0.6\textwidth]{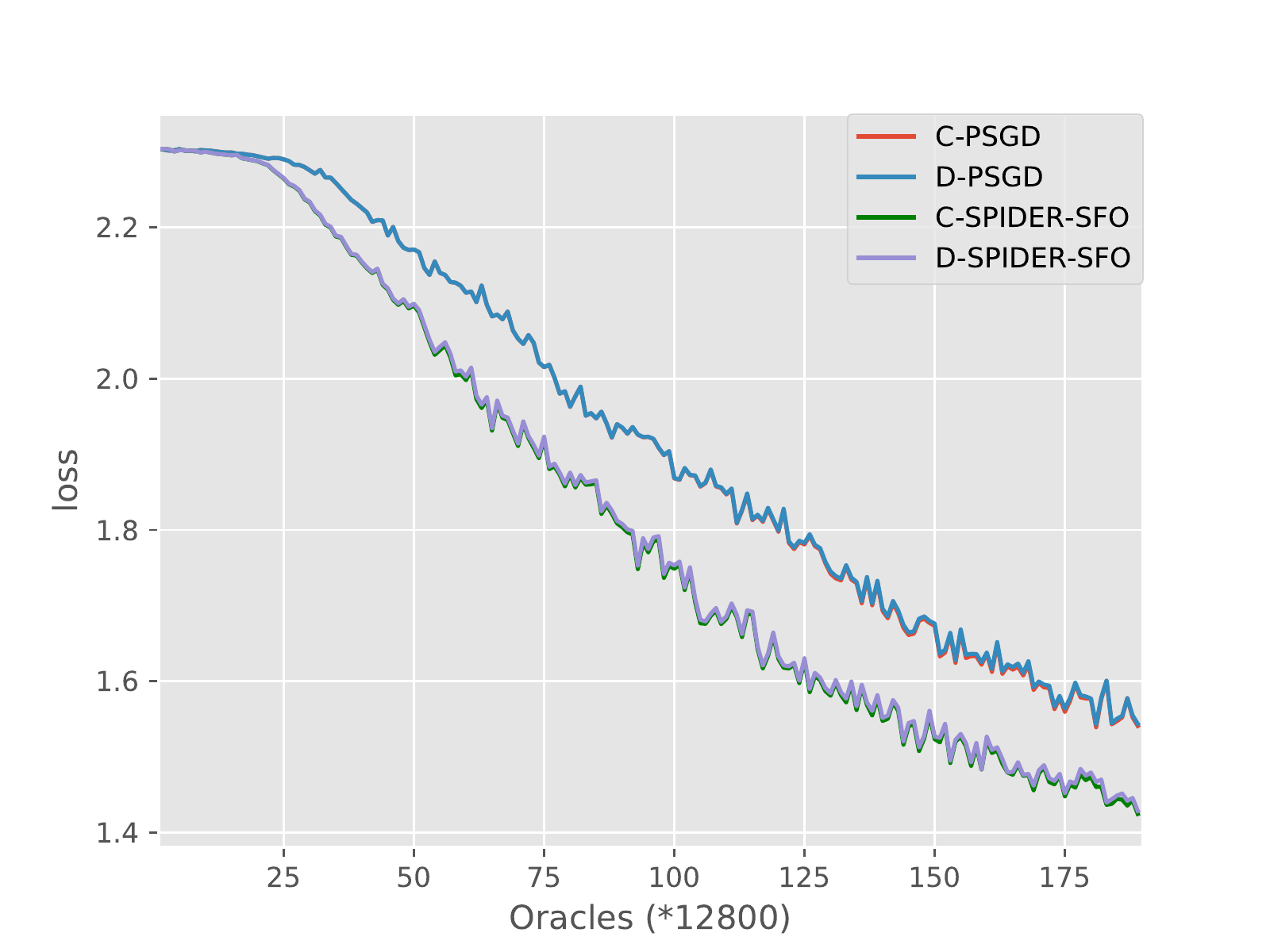}
	\caption{\ref{Fig.main3} is the experiments on CPU cluster with 8 nodes.}
	\label{Fig.main3}
\end{figure*}
\textbf{Hyper-parameters:} For D-PSGD and C-PSGD, we use minibatch of size 128, that is, minibatch of size 16 for each node and tune the constant learning rate $\frac{\eta_0}{\sqrt{K/n}}$. For D-SPIDER-SFO and C-SPIDER-SFO, we set $S_1=256$, $S_2 = 16$, $q=16$ for each node and tune the learning rate for $\{0.1,0.05,0.01,0.005,0.001\}$.

\end{document}